%% file: main.tex
\documentclass[10pt,twocolumn,letterpaper]{article}

\usepackage[pagenumbers]{cvpr}       
\input{preamble}

\definecolor{cvprblue}{rgb}{0.21,0.49,0.74}
\usepackage[pagebackref,breaklinks,colorlinks,citecolor=cvprblue]{hyperref}
\usepackage{graphicx}
\usepackage{amsmath}
\usepackage{amsthm}
\usepackage{amssymb}
\usepackage{booktabs}
\usepackage{multirow}
\usepackage{makecell}
\usepackage{dsfont}
\usepackage{mathrsfs}
\usepackage{pifont}
\usepackage{booktabs}
\usepackage{transparent}
\usepackage{tikz}
\usepackage{pgfplots}
\pgfplotsset{compat=1.17}
\usetikzlibrary{pgfplots.groupplots}
\usepackage{filecontents}

\usepackage{bm}
\usepackage[]{algorithm2e, algpseudocode}
\usepackage{listings}

\usepackage[pagebackref,breaklinks,colorlinks]{hyperref}

\usepackage[capitalize]{cleveref}
\crefname{section}{Sec.}{Secs.}
\Crefname{section}{Section}{Sections}
\Crefname{table}{Table}{Tables}
\crefname{table}{Tab.}{Tabs.}

\makeatletter
\newenvironment{customlegend}[1][]{%
    \begingroup
    \pgfplots@init@cleared@structures
    \pgfplotsset{#1}%
}{%
    \pgfplots@createlegend
    \endgroup
}%

\def\addlegendimage{\pgfplots@addlegendimage}

\def\paperID{15535} 
\def\confName{CVPR}
\def\confYear{2024}
\newcommand{\xmark}{\ding{55}}%

\newcolumntype{H}{>{\setbox0=\hbox\bgroup}c<{\egroup}@{}}

\begin{document}

\title{Fourier-basis functions to bridge augmentation gap: \\ Rethinking frequency augmentation in image classification}
\author{Puru Vaish$^*$ \quad Shunxin Wang$^*$ \quad Nicola Strisciuglio \\
University of Twente\\
{\tt\small \{p.vaish, s.wang-2, n.strisciuglio\}@utwente.nl}
}
\maketitle
\def\thefootnote{*}\footnotetext{Equal contribution}
\begin{abstract}
Computer vision models normally witness degraded performance when deployed in real-world scenarios, due to unexpected changes in inputs that were not accounted for during training. Data augmentation is commonly used to address this issue, as it aims to increase data variety and reduce the distribution gap between training and test data. However, common visual augmentations might not guarantee extensive robustness of computer vision models. In this paper, we propose Auxiliary Fourier-basis Augmentation (AFA), a complementary technique targeting augmentation in the frequency domain and filling the robustness gap left by visual augmentations. 
We demonstrate the utility of augmentation via Fourier-basis additive noise in a straightforward and efficient adversarial setting.
Our results show that AFA benefits the robustness of models against common corruptions, OOD generalization, and consistency of performance of models against increasing perturbations, with negligible deficit to the standard performance of models. It can be seamlessly integrated with other augmentation techniques to further boost performance. 
\vspace*{-\baselineskip}
\end{abstract}

\input{section/introduction}

\input{section/relatedworks}

\input{section/method}

\input{section/experiments}

\input{section/conclusion}

{\small
\bibliographystyle{ieee_fullname}
\bibliography{egbib}
}

\newpage

\appendix
\input{appendix}

\end{document}

%% file: preamble.tex
\usepackage[dvipsnames]{xcolor}


%% file: section/introduction.tex
\section{Introduction}\label{sec:introduction}
Computer vision models usually encounter performance degradation when deployed in real-world scenarios due to unexpected image variations~\cite{hendrycks2019benchmarking,greco2023,Kamann2021}. Improving the robustness of computer vision models to out-of-distribution (OOD) data is thus essential for their reliable practical use.
Among the methods addressing the robustness and generalization of computer vision models~\cite{chen2020simple,Xie_2020_CVPR,Zheng2016Apr,Strisciuglio2020,strisciuglio2022visual,faghri2023reinforce,yucel2023hybridaugment,Gao_2023_ICCV,Hao_2023_WACV}, data augmentation is mostly used for its easy-to-apply characteristics and effectiveness at reducing the distribution gap between training and test data~\cite{Wang2023May}.  
Popular augmentation techniques, such as AugMix~\cite{Hendrycks2019Dec}, AugMax~\cite{Wang2021Oct}, AutoAugment~\cite{cubuk2019autoaugment}, TrivialAugment~\cite{Muller2021Mar}, and PRIME~\cite{Modas2021Dec} have shown great improvements in corruption and perturbation robustness benchmarks and OOD datasets for generalisation, e.g. ImageNet-C, ImageNet-$\mathrm{\Bar{C}}$, ImageNet-3DCC, ImageNet-P, ImageNet-R and ImageNet-v2~\cite{hendrycks2019benchmarking,mintun2021on, kar20223d,hendrycks2021faces,recht2019imagenet}. These approaches mainly focus on adding visual variations to images through random or policy-based combinations~\cite{Hendrycks2019Dec,Muller2021Mar,cubuk2019autoaugment,Liu_2023_ICCV,Liu_2023_CVPR,pmlr-v202-hounie23a,ma2023learning,Marrie_2023_CVPR} of visual transformations aiming at increasing the diversity of training images (expanding on their domain, see visual augmentations in~\cref{fig:teaser}), and adversarial-based augmentations, which address the hardness of training samples but are computationally heavy (see~\cref{tab:comp_eff}, AugMax). 
\begin{figure}
    \centering
    \def\svgwidth{\linewidth}
    \input{section/figs/teaser/teaser}
    \vspace{-1.5em}
    \caption{Frequency augmentation with Fourier-basis functions is complementary to common visual augmentations. They appear \emph{unnatural} and can be used as adversarial examples.}
    \vspace*{-\baselineskip}
    \label{fig:teaser}
\end{figure}
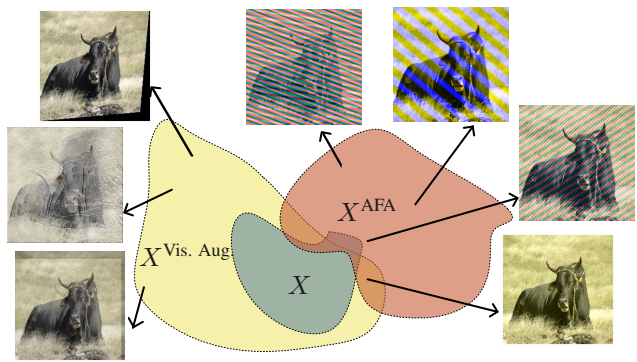%
However, even if trained with visual augmentations, models are still sensitive to image variations not included in the training~\cite{liu2023outofdistribution} and frequency perturbations~\cite{Yin2019Jun}. This occurs due to the pre-defined frequency characteristics of visual transformations, which cannot ensure the complete robustness of models against noise with different frequency characteristics from those encountered during training. Attackers may  exploit this weakness and degrade model performance in operational settings~\cite{9412611}.
This raises a question: \textit{Is there a complementary augmentation technique that can bridge the gap left by visual augmentations?}

Common visual augmentations impact different frequency components of images simultaneously, which are difficult to explicitly control, and might not encompass all possible frequency variations present in unseen corruptions or variantions happening in real-world scenarios~\cite{Saikia_2021_ICCV}.
We thus rethink image augmentation in the frequency domain, and complement visual augmentation strategies with explicit use of Fourier basis functions in an adversarial setting. There has been exploration into frequency-based augmentations to discover capabilities beyond what visual augmentations can achieve. 
~\cite{chen2021amplitudephase,XU2023109474,10190316} swap or mix partial amplitude spectrum between images, aiming to induce more phase-reliance for classification.
\cite{wang2023dfmx} augments images with shortcut features to reduce their specificity for classification. 
AugSVF~\cite{Soklaski2022Feb} introduces frequency noise within the AugMix framework and ~\cite{101007,liu2023improving}  adversarially perturb the frequency components of images. 
These augmentations are computationally heavy, due to the complicated augmentation framework~\cite{Soklaski2022Feb}, computation of multiple Fourier transforms for training images and their augmented versions~\cite{chen2021amplitudephase,XU2023109474,10190316}, identification of learned frequency shortcuts~\cite{wang2023dfmx}, or adversarial training~\cite{101007,liu2023improving}.

In this work, we propose Auxiliary Fourier-basis Augmentation (AFA). We use additive noise based on Fourier-basis functions to augment the frequency spectrum in a more efficient way than other methods that apply frequency manipulations~\cite{chen2021amplitudephase,Soklaski2022Feb,wang2023dfmx}. The effect of additive Fourier-basis functions on image appearance is complementary to those of other augmentations (see~\cref{fig:teaser}). These images can be interpreted as samples representing an adversarial distribution, distinct from those augmented by common visual transformations.
We thus expand upon the conventional idea of adversarial augmentation, moving beyond the generation of imperceptible noise through gradient back-propagation.
We employ a training architecture and strategy with an auxiliary component to address the adversarial distribution, and a main component for the original distribution, similarly to AugMax~\cite{Wang2021Oct}.
However, the adversarial distribution that we construct using additive Fourier-basis is much less computationally expensive than that of AugMax (and other visual augmentation methods - see~\cref{tab:comp_eff}). It contributes to comparable or higher generalization results, while allowing for the training of larger models on larger datasets (e.g. ImageNet).
Our contributions are:
\begin{itemize}
    \item We propose a straightforward and computationally efficient augmentation technique called AFA. We show that it enhances robustness of models to common image corruptions, improves OOD generalization  and consistency of prediction w.r.t. perturbations;
    \item  We expand the augmentation space, complementary to that of visual augmentations, by exploiting amplitude- and phase-adjustable frequency noise, and use it in an adversarial setting. Our method reduces the augmentation gap of common visual augmentations.
\end{itemize}

\setlength{\tabcolsep}{1.5pt} 
\renewcommand{\arraystretch}{0.75} 
\begin{table}
\scriptsize
    \centering
    \begin{tabular}{lHllllllll}
    \toprule
        & APR-SP & \makecell{AFA (ours) \\ w/o aux.} & \makecell{AFA \\ (ours)} & AugMix$^\dagger$ & \makecell{AFA \\ w/ AugMix} & PRIME & \makecell{AFA \\ w/ PRIME} & \makecell{AugMax$^\dagger$}\\ \toprule
        FLOPs & $\times 1$ & $\times 1$ & $\times 2$ & $\times 3$ & $\times 2$ & $\times 1$ & $\times 2$ & $\times 8$\\
        Memory & $\times 1.02$ & $\times 1.02$ & $\times 1.62$ & $\times 2.66$ & $\times 1.83$ & $\times 2.50$ & $\times 3.06$ & $\times 2.35$\\
        \bottomrule
    \end{tabular}
    \vspace{-2mm}
    \caption{AFA adds minimal computational burden to existing methods and is more efficient compared to other adversarial methods. It requires only $\times 1.62$ memory and just $\times 2$ the FLOPs of standard augmentation~\cite{he2015deep} training whereas AugMax uses $\times 2.35$ the memory and $\times 8$ the FLOPs when using $5$ PGD steps. Methods with $^\dagger$ denote the use of loss with JSD.}
    \label{tab:comp_eff}
    \vspace*{-2\baselineskip}
\end{table}
\setlength{\tabcolsep}{6pt} 
\renewcommand{\arraystretch}{1} 

%% file: section/figs/teaser/teaser.tex
\begingroup%
  \makeatletter%
  \providecommand\color[2][]{%
    \errmessage{(Inkscape) Color is used for the text in Inkscape, but the package 'color.sty' is not loaded}%
    \renewcommand\color[2][]{}%
  }%
  \providecommand\transparent[1]{%
    \errmessage{(Inkscape) Transparency is used (non-zero) for the text in Inkscape, but the package 'transparent.sty' is not loaded}%
    \renewcommand\transparent[1]{}%
  }%
  \providecommand\rotatebox[2]{#2}%
  \newcommand*\fsize{\dimexpr\f@size pt\relax}%
  \newcommand*\lineheight[1]{\fontsize{\fsize}{#1\fsize}\selectfont}%
  \ifx\svgwidth\undefined%
    \setlength{\unitlength}{469.91271973bp}%
    \ifx\svgscale\undefined%
      \relax%
    \else%
      \setlength{\unitlength}{\unitlength * \real{\svgscale}}%
    \fi%
  \else%
    \setlength{\unitlength}{\svgwidth}%
  \fi%
  \global\let\svgwidth\undefined%
  \global\let\svgscale\undefined%
  \makeatother%
  \begin{picture}(1,0.65063646)%
    \lineheight{1}%
    \setlength\tabcolsep{0pt}%
    \put(0,0){\includegraphics[width=\unitlength,page=1]{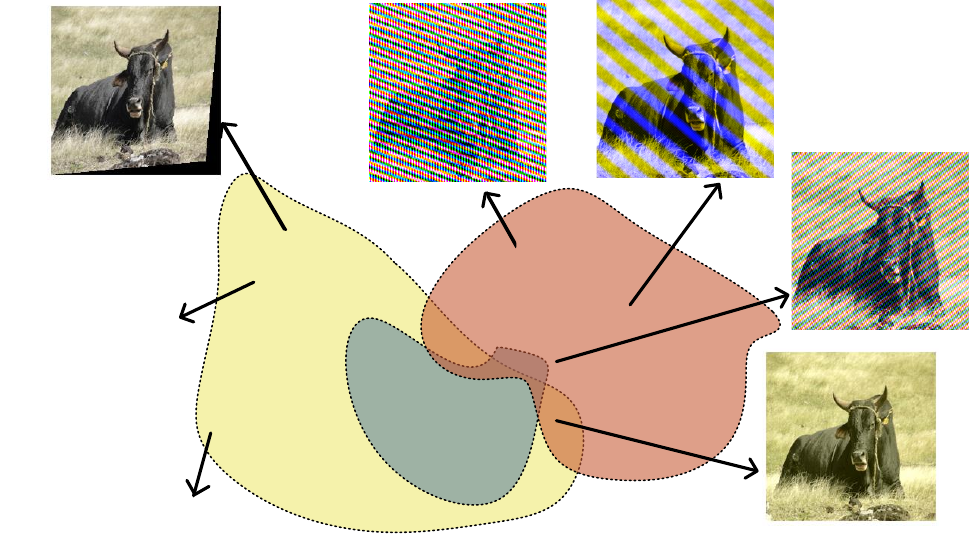}}%
    \put(0.525,0.22731628){\color[rgb]{0.10196078,0.10196078,0.10196078}\makebox(0,0)[lt]{\lineheight{1.25}\smash{\begin{tabular}[t]{l}$X^{\text{AFA}}$\end{tabular}}}}%
    \put(0.44690027,0.1097097){\color[rgb]{0.10196078,0.10196078,0.10196078}\makebox(0,0)[lt]{\lineheight{1.25}\smash{\begin{tabular}[t]{l}$X$\end{tabular}}}}%
    \put(0.21,0.155){\color[rgb]{0.10196078,0.10196078,0.10196078}\makebox(0,0)[lt]{\lineheight{1.25}\smash{\begin{tabular}[t]{c} $X^{\text{Vis. Aug.}}$ \end{tabular}}}}%
    \put(0,0){\includegraphics[width=\unitlength,page=2]{section/figs/teaser/teaser_ink.pdf}}%
  \end{picture}%
\endgroup%

%% file: section/relatedworks.tex
\section{Related works}\label{sec:relatedworks}
Data augmentation includes a set of techniques to increase data variety, thus reducing the distribution gap between training and test data. Generalization and robustness performance of models normally benefits from the use of data augmentation for training~\cite{Wang2023May} or at test-time~\cite{NEURIPS2020_2ba59664}. 

\noindent \textbf{Image-based augmentations.} Common image augmentation techniques include transformations, e.g. cropping, flipping, rotation, among others~\cite{Wang2023May}. Applying the transformations with fixed configuration lacks flexibility when the models encounter more variations in the inputs at testing time. Thus, algorithms were designed to combine transformations randomly, e.g. AugMix~\cite{Hendrycks2019Dec}, RandAug~\cite{NEURIPS2020_d85b63ef}, TrivialAugment~\cite{Muller2021Mar}, MixUp~\cite{zhang2018mixup}, and CutMix~\cite{yun2019cutmix}. However, random combinations  might not be optimal. In~\cite{cubuk2019autoaugment}, AutoAugment was proposed, based on using reinforcement learning  to find the best policy on how to combine basic transformations for augmentation. AugMax~\cite{Wang2021Oct} instead combines transformations adversarially, aiming at complementing augmentations based on diversity with others that favour hardness of training data. PRIME~\cite{Modas2021Dec} samples transformations with maximum-entropy distributions. ~\cite{Suzuki_2022_CVPR} augments images based on knowledge distilled by a teacher model.  However, these approaches address variations limited by visually-plausible transformations only.

\noindent \textbf{Frequency-based augmentations.} In~\cite{Yin2019Jun}, it was discovered that models trained with visual transformations might be vulnerable to noise impacting certain parts of the frequency spectrum (e.g. high-frequency components), demonstrating that visual augmentations do not completely guarantee robustness. Complementary augmentation techniques are thus required to fill the augmentation gap left by visual augmentations. The straightforward approach is augmentation in the frequency domain.
For example,~\cite{chen2021amplitudephase} mixes the amplitude spectrum of images to reduce reliance on the amplitude part of the spectrum and induce phase-reliance for classification. ~\cite{XU2023109474,10190316} swap or mix the amplitude spectrum of images. ~\cite{wang2023dfmx} augments images with shortcut features to reduce their specificity for classification, mitigating frequency shortcut learning. 
~\cite{Soklaski2022Feb} introduces frequency noise in the AugMix framework.
~\cite{101007978,liu2023improving} adversarially perturb images in the frequency domain.
While these techniques address what visual augmentations may overlook, they also have limitations.
Most frequency augmentation methods  are based on manipulation of the frequency components of images. They usually have high computational requirements to identify frequency shortcuts~\cite{wang2023dfmx} (f.i. using~\cite{Wang2023ICCV,wang2022frequency}), implement adversarial training setup~\cite{liu2023improving} or calculate multiple Fourier transforms of original and augmented images~\cite{wang2023dfmx,chen2021amplitudephase,XU2023109474,10190316}.
 
We instead propose to use Fourier-basis functions as additive noise in the frequency domain. Our augmentation technique requires only one extra step during training rather than multiple pre-processing and expensive computations during training time as in other methods~\cite{wang2023dfmx,chen2021amplitudephase,XU2023109474,10190316}, and works to complement image-based augmentations. Furthermore, we simplify the adversarial training framework of AugMax~\cite{Wang2021Oct}, not requiring an optimization process to maximize the hardness of adversarial augmentation, and achieving comparable or higher robustness. This allows the use of adversarial augmentations at larger-scale.
We account for the induced distribution shifts in the frequency domain via an auxiliary component. The benefit of AFA is complementary to visual augmentations, and we can incorporate them seamlessly to further boost model robustness.

%% file: section/method.tex
\section{Preliminary: Fourier-basis functions}\label{par:fb_prel}
We utilize Fourier-basis functions in our augmentation strategy as an additive perturbation to the images. They are sinusoidal wave functions used as basic components of the Fourier transform to represent signals and images.
A real Fourier basis function has two parameters, namely a frequency $f$ and direction $\omega$, and is denoted as:
\vspace{-0.5\baselineskip}
\begin{equation}\label{eqn:general_wave}
    A_{f, \omega}(u, v) = R\sin(2\pi f(u\cos(\omega) + v\sin(\omega) - \pi / 4)),
\vspace{-0.5\baselineskip}
\end{equation}
\noindent where \(A_{f, \omega}(u, v)\) represents the amplitude of the wave at position \((u, v)\). The function involves the sine of a 2D spatial frequency \(2\pi f\) to produce a planar wave with a specific frequency \(f\), and angle \(\omega\) that indicates the direction of propagation. $R$ is chosen such that the planar wave has unit $l_2$-norm.
A particular Fourier basis function, characterized by specific frequency ($f$) and direction ($\omega$), can be associated with a Dirac delta function in the spectral domain. Therefore, when employed in an additive manner, as in our augmentation strategy, this Fourier-basis function facilitates the targeted modification of particular frequency components of images.
Examples of Fourier-basis waves superimposed on images are shown in~\cref{fig:fb_example}.

\begin{figure}[!h]
    \centering
    \includegraphics[width=\linewidth]{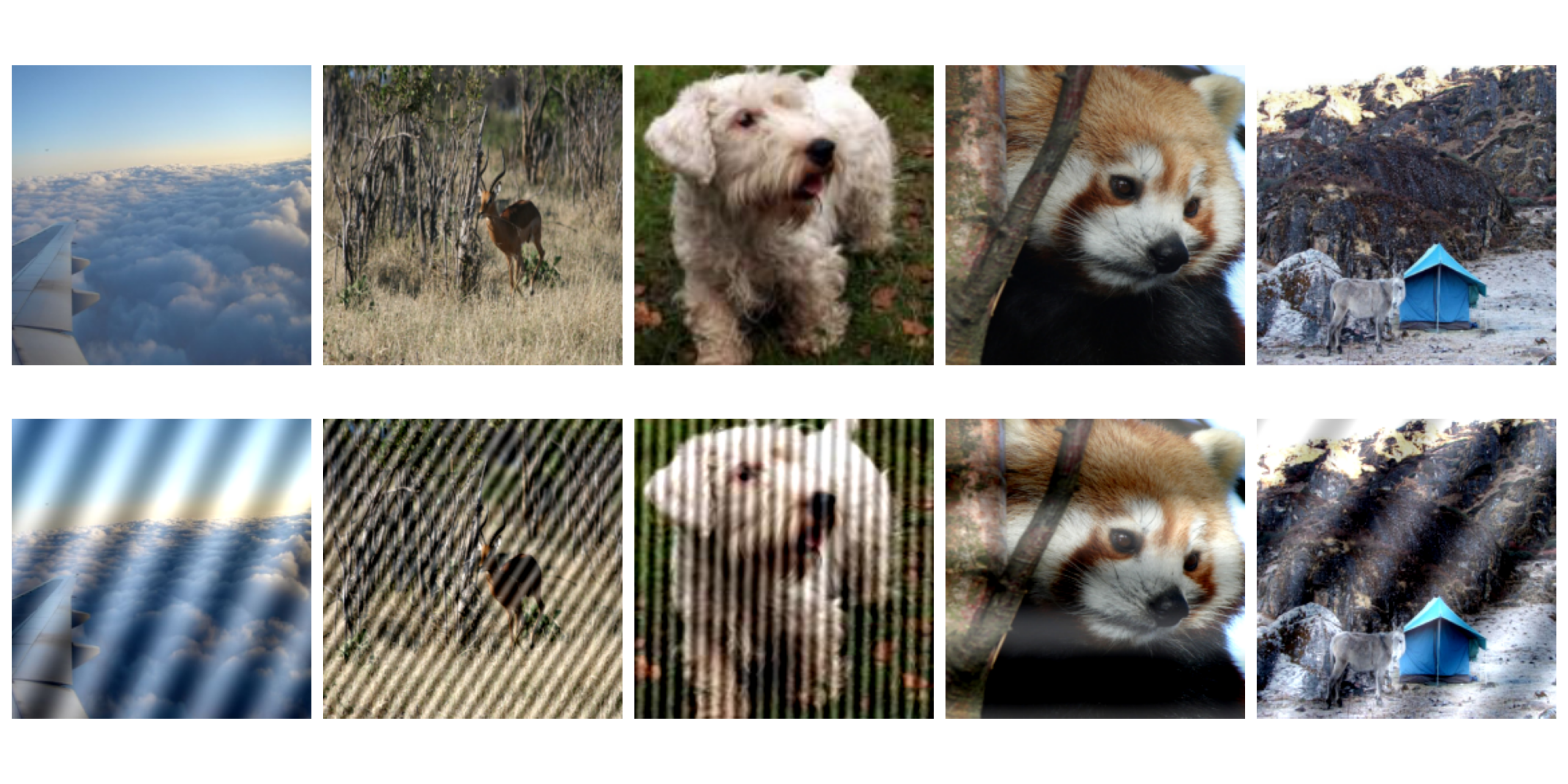}
    \vspace{-2.5em}
    \caption{Example of Fourier-basis functions  added to natural images. They appear as \emph{gratings} that obscure spatial information.}
    \vspace*{-\baselineskip}
    \label{fig:fb_example}
\end{figure}

\section{Auxiliary Fourier-basis Augmentation}

The Auxiliary Fourier-basis Augmentation (AFA)  that we propose is based on two lines of augmentations, one considered in-distribution (using visual augmentations) and another considered out-of-distribution or adversarial (using frequency-based noise) as shown in~\cref{fig:schema}. We generate the adversarial augmented images by sampling a Fourier-basis and a strength parameter per colour channel, and adding them to the original images. Visually augmented and adversarially augmented training images are then processed using a main component and an auxiliary component, respectively. Joint optimisation of two cross-entropy functions encourages robust and consistent classification, as it promotes correctness under adversarially augmented images. Details of the different parts of the method are reported below.

\begin{figure*}[!htb]
    \centering
    \def\svgwidth{\linewidth}
    \input{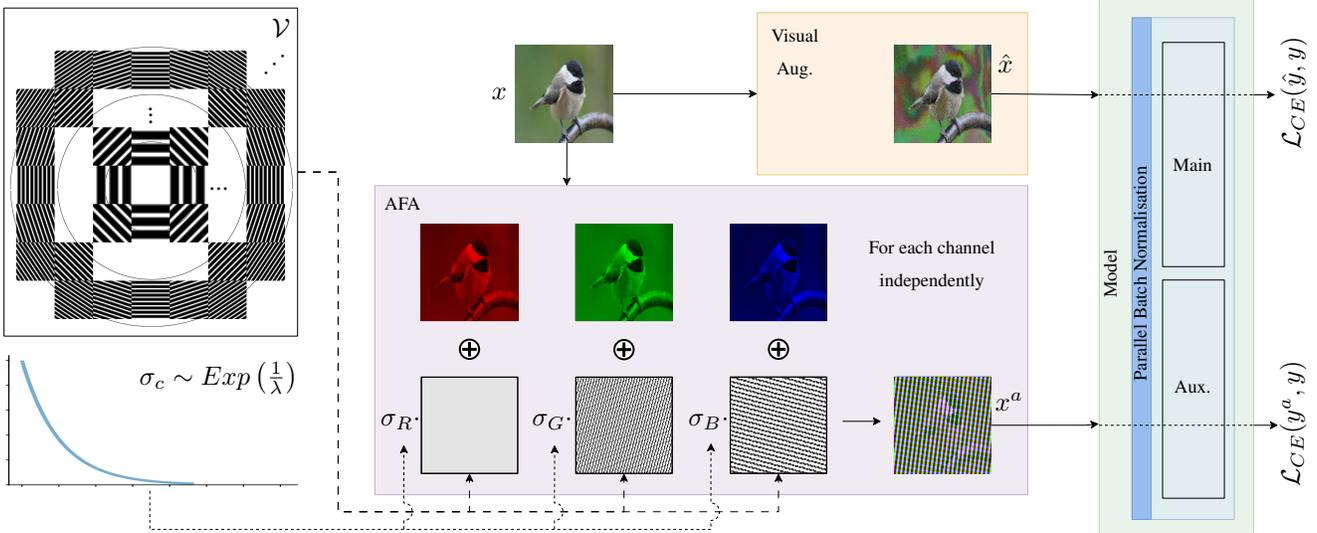}
    \caption{Schema of the AFA augmentation pipeline. The image $x$ is augmented using AFA, which adds a planar wave per channel $c$ of the image at a strength value $\sigma_c$ sampled from an exponential distribution (eq.\ref{eqn:afa}). The AFA augmented image $x^a$ is used for training, processed through the auxiliary component of the parallel batch normalisation layer (for models that use batch normalization  to track batch statistics, e.g. ResNet).
    Other visual augmentations are applied in parallel, and used for training via the main component of the normalization layer. Finally, we train via optimizing two cross-entropy losses, one for the main and the other for the auxiliary component.}
    \vspace{-\baselineskip}
    \label{fig:schema}
\end{figure*}

\noindent \textbf{Generation of adversarial augmented images. }
Randomly sampling augmentations and applying them to images with random strengths  was shown to be sufficient to outperform more complex strategies~\cite{Muller2021Mar}. 

We follow this design principle in our method to generate adversarial augmented images with Fourier basis functions, which allows us to avoid optimization steps to determine the worst-case combination of augmentations as in AugMax~\cite{Wang2021Oct}. We produce adversarial augmented images by adding a different Fourier basis function $A_{f,\omega}$ per channel of the original RGB image. We generate the Fourier basis functions by sampling $f$ and $\omega$ from uniform distributions as $f\sim \mathcal{U}_{[1,\text{M}]}$ and $\omega\sim \mathcal{U}_{[0,\pi]}$, where M is the image size. The sampling space of all Fourier-basis is denoted as $\mathcal{V}$. We add the generated Fourier basis functions per channel $c$ with a weight factor sampled from an exponential distribution ${\sigma_c~\sim~\text{Exp}(1 / \lambda)}$, with $c \in \{\text{R}, \text{G}, \text{B}\}$. 
The selection of the exponential distribution for sampling augmentation magnitude is motivated by the concept of event rate, where perturbations with larger magnitudes become progressively less likely, albeit still possible. This is controlled by adjusting $\lambda$, ensuring a balance between maintaining diversity in sampled values while minimizing the occurrence of extremely large augmentation perturbations. In~\cref{par:strength_abl}, we show how the parameter $\lambda$ affects the augmentation results.

The proposed augmentation process results in a 3-channel  image $x^a = [x^a_R, x^a_G, x^a_B]$, where:
\begin{align}\label{eqn:afa}
    x^a_c = \text{Clamp}_{[0, 1]}(x_c + \sigma_c A_{f_c, \omega_c}), && c \in \{\text{R}, \text{G}, \text{B}\}.
\end{align}
An example of image $x^a$ augmented with additive Fourier-basis functions is shown in our method schema in~\cref{fig:schema}.
The augmentation also results in the augmentation at exactly the particular sampled frequency and phase in the Fourier domain (Appendix~\ref{app:proof}) therefore being effective at perturbing the samples in Fourier domain with precision.

\noindent \textbf{Auxiliary component for distribution shifts. }
As shown in~\cref{fig:fb_example,fig:schema}, the Fourier-basis augmentations result in images with an unnatural appearance due to substantial frequency perturbations.
The presence of planar waves across the augmented images determines the \emph{unnaturalness} of image appearance, which can  be seen as adversarial attacks on the images (experimentally verified in Appendix~\ref{app:pca_vis}). These augmentations disrupt the learned mean and variance in batch normalization layers, which are inconsistent with the distribution shifts induced by our augmentation and lead to inconsistent activations (Appendix~\ref{app:bn_diff}). This results in a negative impact on model convergence and generalization abilities. 

We address these issues by deploying architectural components in the training, capable of handling distribution shifts explicitly by tracking statistics and adjusting the loss function accordingly. Namely, we incorporate auxiliary components into the model, such as Parallel Batch Normalization layers and an additional cross-entropy term in the loss function to specifically account for these adversarial augmented images.
These modifications to the model architecture and training enhance  performance, particularly in the presence of distribution shifts, contributing to better generalization, robustness to common corruptions and consistency to time-dependent increasing perturbations. The introduction of parallel batch normalization layers is motivated by the need to account for distribution shifts induced by adversarial (Fourier-basis) augmentations, as observed in~\cite{Wang2021Oct}. With the parallel batch normalisation, the affine parameters and statistics of main and auxiliary distributions are recorded separately. This allows independent learning of distribution of the visually and adversarially augmented images. Without these additional normalization layers, the model training assumes a single-modal sample distribution, limiting its ability to differentiate between the main and the adversarial distribution, thus negatively affecting overall performance. In~\cref{par:aux_abl}, we show the result of not employing the auxiliary components.

It is worth noting that for models that do not employ batch normalization layers (e.g. CCT that uses layer normalization and does not track statistics), the parallel normalization layers are not needed. However, the extra term in the loss function (see next paragraph) to generate consistent predictions across distribution shifts serves as a regularization mechanism (verified in Appendix~\ref{app:reg_effect}).

\noindent \textbf{Loss function. }
We work in the supervised learning setting with a training dataset $\mathcal{D}$ consisting of clean images $x$ with labels $y$. We train the model in the main architecture stream (see~\cref{fig:schema}) using a cross-entropy loss $\mathcal{L}_{\text{CE}}(\hat{y}, y)$, where $y$ is the ground-truth label and $\hat{y}$ is the predicted label for images augmented with a given visual augmentation strategy (e.g. standard, PRIME, etc.). Under the non-auxiliary setting, models thus optimise the standard cross entropy loss.

In the auxiliary setting, we add an extra cross-entropy loss term $\mathcal{L}_{\text{CE}}(y^{a}, y)$, which optimise the model to predict the correct label on adversarial augmented images whose predicted label is denoted by $y^a$, contributing to robustness of the model w.r.t. aggressive distribution shifts. 
We refer to the combined loss function $\mathcal{L}_{\text{ACE}}$, taking the average of the two cross-entropy terms, as the Auxiliary Cross Entropy (ACE) Loss:
\vspace{-0.5\baselineskip}
\begin{equation}\label{eqn:adv}
    \mathcal{L}_{\text{ACE}}(\hat{y}, y^a, y) =  \frac{1}{2} \left[ \mathcal{L}_{\text{CE}}(\hat{y}, y) + \mathcal{L}_{\text{CE}}(y^{a}, y) \right].
\vspace{-0.5\baselineskip}
\end{equation}

It contributes to achieve comparable performance, with lower training time and complexity, than using the Jensen-Shannon Divergence (JSD) loss~\cite{Hendrycks2019Dec, Wang2021Oct}. Our motivation to not employ the JSD loss is the reduced training time due to less computational complexity. In our experiments, for comparison purposes, we also use the JSD loss in the auxiliary setting, where training batches are augmented using AFA and go through auxiliary components. We report results in~\cref{sec:ablation} (\cref{fig:with_without_jsd}).

%% file: section/experiments.tex
\section{Experiments and results}
We compare AFA with other popular augmentation techniques, evaluating robustness to common corruptions, generalization abilities and consistency to time-dependent increasing perturbations, on benchmark datasets.

\subsection{Experiment setup}
\noindent \textbf{Datasets.} We trained models on the CIFAR-10 (C10)~\cite{cifar10}, CIFAR-100 (C100)~\cite{cifar100}, TinyImageNet (TIN)~\cite{tin} and ImageNet (IN)~\cite{5206848} datasets and evaluate them on the corresponding robustness benchmark datasets, namely C10-C, C100-C, TIN-C, IN-C~\cite{hendrycks2019benchmarking}, IN-$\Bar{\text{C}}$~\cite{mintun2021on}, and IN-3DCC~\cite{kar20223d}. For ImageNet-trained models,  we further evaluate their generalisation performance on the IN-v2~\cite{recht2019imagenet} and IN-R datasets~\cite{hendrycks2021faces}, and consistency of performance on time-dependent increasing perturbations on the IN-P dataset~\cite{hendrycks2019benchmarking}.

\noindent \textbf{Architectures and training details. }
We train ResNet~\cite{he2015deep} and Compact Convolution Transformers (CCTs)~\cite{hassani2022escaping}. We train ResNet-18 and CCT-7/3x1 (32 resolution) on C-10, C-100, and only ResNet-18 on TIN. In the case of ImageNet, we train ResNet-18, ResNet-50 and CCT-14/7x2 (224 resolution). Under auxiliary setting, we use the DuBIN variant of ResNet~\cite{Wang2021Oct}. We always use standard transforms~\cite{he2015deep} before other augmentations. 
Implementation details and hyperparameter configurations are in Appendix~\ref{app:impl_details}. We release code and  models\footnote{Code and models: \href{https://github.com/nis-research/afa-augment}{https://github.com/nis-research/afa-augment}}.

\noindent \textbf{Evaluation metrics.}
We evaluate the classification accuracy on the original test set, which we refer to as standard accuracy (SA), and the average classification accuracy over all corruptions in the robustness benchmarks as robustness accuracy (RA). This provides direct comparison between model performance on original and corruption benchmark datasets. We also compute the mean corruption error (mCE)~\cite{hendrycks2019benchmarking} for TIN and IN (for CIFAR there are no baselines advised) to evaluate the normalized robustness of models against image corruptions, the mean flip rate (mFR) and the mean top-5 distance (mT5D) to evaluate the consistency performance of models against increasing perturbations. For the evaluation of generalization performance, we compute the accuracy on the ImageNet-R and ImageNet-v2 test sets (note that ImageNet-v2 has 3 test sets, and we report the average accuracy on them). More details about the metrics are in Appendix~\ref{app:expl_metrics}. 

\subsection{Results}
\input{section/tables/tfba_augmax_comp}
\input{section/tables/in/resnet}

\noindent \textbf{Comparison with AugMax.}
We first report a direct comparison with AugMax~\cite{Wang2021Oct} in~\cref{tab:tfba_augmax_comp}, as AFA addresses the computational shortcomings of generating adversarial augmentations via PGD iterations, and of using a JSD loss for alignment of the distribution of original and (adversarially) augmented images. We use AugMix as main augmentation, as in AugMax, and ablate on the use of JSD and ACE loss. 

We show that AFA achieves comparable (or better) performance than AugMax, despite it being much less computational intensive. We indeed demonstrate that we can generate adversarial augmentations by only adding (weighted) Fourier-basis waves per color channel, not requiring PGD steps, and can train the models using an extra cross-entropy instead of the expensive JSD loss.
The improvements granted by our approach are particularly evident in the case of ImageNet (using ACE), where we gain $1.6\%$ of standard accuracy and $4.1\%$ of robust accuracy ($5.6\%$ mCE) performance w.r.t. AugMax. 
Considering the increased computational efficiency and the simplicity of adversarial augmentation method, AFA is a more versatile and effective tool than AugMax. Hence, in the rest of the paper, we do not report further results of the AugMax framework, due to its high computational requirements, which complicate the training of larger models (e.g. ResNet-50 and CCT).

\noindent \textbf{Robustness, generalization and consistency.}
In~\cref{tab:imagenet_exps}, we report results achieved by AFA combined with different visual augmentation methods, AugMix, PRIME, TrivialAugment (TA), to train different architectures (ResNet, CCT). We evaluate robustness to common corruptions on IN-C, IN-$\bar{\text{C}}$ and IN-3DCC, OOD generalisation on IN-v2 and IN-R, and consistency w.r.t. increasing perturbations on IN-P.  

AFA generally contributes to a boost of performance (green colored results in~\cref{tab:imagenet_exps}) when combined with different visual augmentation techniques, reducing the robustness and generalization gap for different model architectures.
When compared to another Fourier-based augmentation technique, APR-SP~\cite{chen2021amplitudephase}, AFA outperforms it on all benchmarks when trained with only standard augmentation techniques.
Also in the case of AugMix and AFA, we record better overall performance over AugMix alone, when using both the ACE loss and loss with JSD.
For the transformer architecture CCT, training with AFA contributes to an even stronger improvement in all tests.
These results stay consistent for smaller resolution datasets (CIFAR and TIN), as we report at the end of this section.

\begin{figure*}[!t]
    \centering
    \footnotesize
    \input{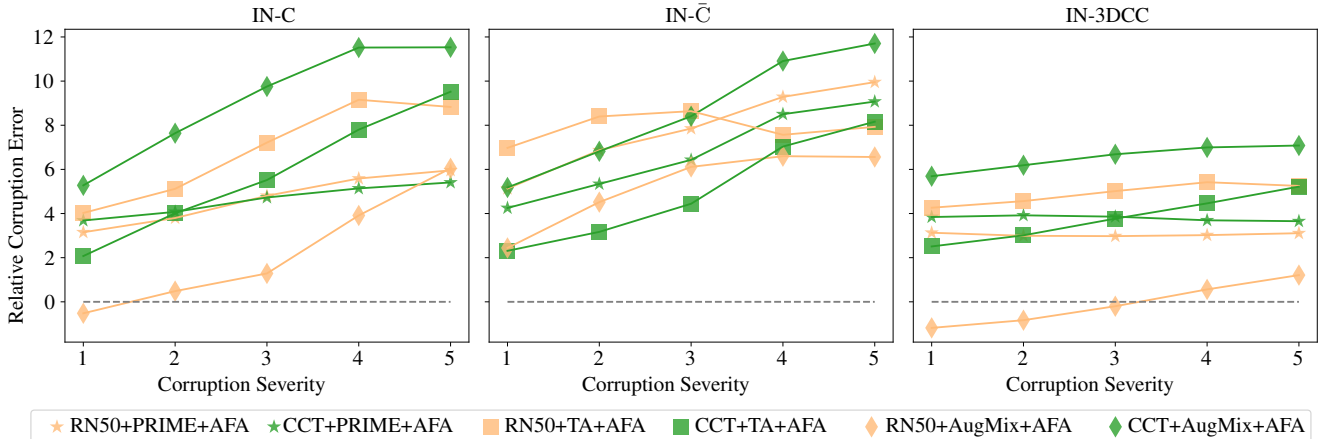}
    \vspace{-2em}
    \caption{Relative error per corruption severity, computed by subtracting the classification error of models trained with PRIME, TrivialAugment, and AugMix with that of corresponding models trained with PRIME+AFA, TrivialAugment+AFA, and AugMix+AFA. }
    \label{fig:plot_rn18}
\end{figure*}

\noindent \textbf{Robustness to high-severity corruptions.}
AFA contributes to a consistent improvement of robustness of models at increasing corruption severity.
We compute the relative corruption error, namely the difference between the corruption error of models trained with a visual augmentation technique only and those trained with both visual augmentations and AFA, and report it in ~\cref{fig:plot_rn18} for different corruption severity.
A positive value indicates that models trained with the addition of AFA have better robustness.
For higher corruption severity, AFA contributes to stronger robustness, measured by an increase in the relative corruption error in~\cref{fig:plot_rn18}. 
The improvements obtained by AFA on IN-3DCC are slightly less pronounced than those on IN-C and IN-$\mathrm{\Bar{C}}$. This is attributable to the specific corruptions in IN-3DCC that concern 3D geometric information, and are somewhat more complicated  image transformations. However, AFA contributes to a substantial improvement w.r.t. to models trained without it.
We thus highlight that AFA is very beneficial for increasing robustness to aggressive corruptions of the test images. Details of the results at different severity are in Appendix~\ref{app:per_severity}.

\noindent \textbf{Fourier heatmap: robustness in the frequency spectrum.} 
We further evaluate the robustness of models to perturbations at specific frequencies, using test images perturbed with frequency noises according to~\cite{Yin2019Jun}. We present the results in the form of Fourier heatmaps, see~\cref{fig:fheats} for heatmaps of ResNet18 models (trained on ImageNet), and in Appendix~\ref{app:heatmap_cct} for the heatmaps of CCT models. 
The intensity of a pixel at location $(u,v)$ in the heatmap indicates the classification error of a model tested on images perturbed by Fourier noise at frequency $(u,v)$ in the frequency spectrum (implementation details are in Appendix~\ref{app:expl_metrics}). 
ResNet18 trained with standard augmentations setting (baseline) is very sensitive to perturbations at low and middle-high frequency (see~\cref{fig:fheats}), while those trained with visual augmentations like PRIME and TrivialAugment (TA) still show vulnerability at low and middle-high frequency noise.
When training models with AFA, i.e. PRIME+AFA and TA+AFA, the models become more robust to frequency pertubations, especially at middle-high frequency. AFA can provide extensive robustness to frequency perturbations and bridge the robustness gap that visual augmentation might not cover. 

\begin{figure}[!t]
    \centering
    \input{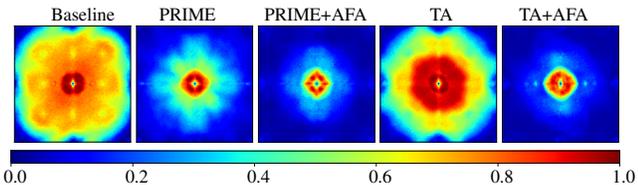}
    \vspace{-2em}
    \caption{Fourier heatmaps of ResNet18 trained with standard setup, and PRIME and TrivialAugment, with and without AFA. }
    \label{fig:fheats}
\vspace{-1\baselineskip}
\end{figure}

\noindent \textbf{Results on CIFAR and TIN.}
In~\cref{tab:low_res_exps}, we present the robustness results on smaller resolution datasets, C10 and C100. The results on TIN are in Appendix~\ref{app:tin_results}. There results are inline with those reported on IN in~\cref{tab:imagenet_exps}.

\input{section/tables/c10/resnet18_cifar10}

\subsection{Ablation}
\label{sec:ablation}

\paragraph{Auxiliary components.}\label{par:aux_abl}
We investigate the contribution and importance of the auxiliary components in improving model robustness. We trained models with AFA-augmented images, passing through only the main components or the auxiliary components. The results in~\cref{tab:afa_aux_ablation_comp}, i.e. lower RA and higher mCE of models trained with AFA applied only in the main components, highlight the importance of AFA auxiliary components. The auxiliary components play a crucial role in mitigating the impact of aggressive adversarial distribution shifts induced by AFA. By doing so, they contribute to model ability to learn from the original distribution, while AFA facilitates learning robustness to distribution shifts.
This is also highlighted in the substantial decrease in SA for models not employing auxiliary components.
While model robustness  improves under both settings, the performance gain for the auxiliary setting is three to five percentage points higher across all datasets.

\input{section/tables/importance_of_aux}

\noindent \textbf{ACE vs JSD.}
As part of our method, we replaced the use of JSD with ACE which is less computationally burdening. We thus performed an ablation analysis of the tradeoff of using JSD. 
We report results for robustness using mCE and Robust Accuracy (RA) in~\cref{fig:with_without_jsd}, and observe that JSD does not significantly improve the robustness of our model to image corruptions, despite it being more computationally heavy than using ACE.
Using JSD also results in slightly worse robustness on C100. Given the minimal differences, we opt for the simpler ACE loss for training with the AFA augmentation pipeline and only using JSD if other techniques (e.g. AugMix) employ them.

\noindent \textbf{Effect of hyperparameter $1 / \lambda$.}\label{par:strength_abl}
We studied also the contribution of the mean $1 / \lambda$ of the exponential distribution that we use to sample the weight factor for the channel-wise application of the Fourier-basis augmentations. We provide the results in~\cref{fig:strength_ablation}, and observe that our method has low sensitivity to the choice of the rate parameter. This is attributable to the choice of the exponential distribution that allows larger values to be sampled even if they are less likely. We indeed observe that larger values of $1 / \lambda$, which result in larger perturbations (in the range of 10 to 15), result in stronger gains in robustness.
At the same time, there is no clear trend in the standard accuracy on the clean dataset, with only minimal variations for the larger values, indicating that the choice of the $1 / \lambda$ value does not have a specific influence on the correct functioning of AFA.
\pgfplotsset{
    every non boxed x axis/.style={} 
}

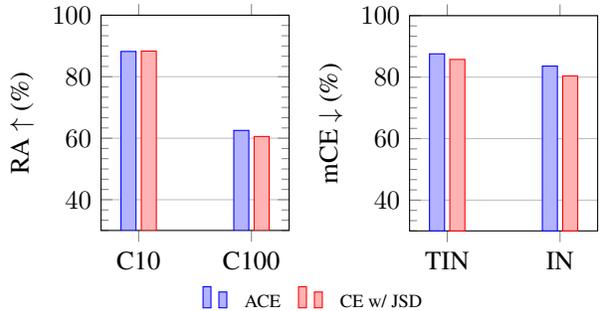
\begin{figure}[!t]

\begin{subfigure}[b]{0.45\linewidth}
\begin{tikzpicture}
\pgfplotstableread[col sep=comma]{data_c.csv}\datatable

\begin{axis}[
    ybar,
    x=1.5cm,
    enlarge x limits = {abs=0.5 cm},
    ymin=30, ymax=100,
    width=\linewidth,
    height=1.75in,
    bar width=0.2cm,
    ylabel={RA $\uparrow$ (\%)},
    symbolic x coords={C10,C100},
    xtick=data,
    minor y tick num=5,
    ymajorgrids,
    x tick label style={rotate=0, anchor=north},
]
\addplot table [x=Dataset, y=CE] {\datatable};
\addplot table [x=Dataset, y=CE_JSD] {\datatable};

\legend{\scriptsize ACE, \scriptsize CE w/ JSD}
\legend{};
\end{axis}
\end{tikzpicture}
\end{subfigure}%
\hspace{1em}%
\begin{subfigure}[b]{0.45\linewidth}
\begin{tikzpicture}
\pgfplotstableread[col sep=comma]{data_in.csv}\datatable
\begin{axis}[
    ybar,
    x=1.5cm,
    enlarge x limits = {abs=0.5 cm},
    ymin=30, ymax=100,
    width=\linewidth,
    height=1.75in,
    bar width=0.2cm,
    ylabel={mCE $\downarrow$ (\%)},
    symbolic x coords={TIN,IN},
    xtick=data,
    minor y tick num=5,
    ytick pos=left,
    ymajorgrids,
    x tick label style={rotate=0, anchor=north},
]
\addplot table [x=Dataset, y=CE] {\datatable};
\addplot table [x=Dataset, y=CE_JSD] {\datatable};
\legend{\scriptsize ACE, \scriptsize CE w/ JSD}
\legend{};
\end{axis}
\end{tikzpicture}%
\end{subfigure}%

\begin{subfigure}[b]{\linewidth}
    \centering
    
    \begin{tikzpicture}
        \begin{customlegend}[legend columns=-1,legend style={draw=none,column sep=1ex},legend entries={\scriptsize ACE, \scriptsize CE w/ JSD}]
        \addlegendimage{blue,fill=blue!30!white,ybar,ybar legend}
        \addlegendimage{red,fill=red!30!white,ybar,ybar legend}
        \end{customlegend}
    \end{tikzpicture}
\end{subfigure}
    \vspace{-2em}
    \caption{Comparison of using objective with and without the JSD term. All models are ResNet-18 trained with only AFA in the auxiliary component and no other augmentations. When used with JSD two batches passed through Auxiliary components and there was no main augmentation (in total 3 batches, 1 clean and 2 AFA).}
    \label{fig:with_without_jsd}
\end{figure}

\begin{figure}[!t]
\footnotesize
\begin{tikzpicture}        
\begin{groupplot}[
    group style={
        group name=my fancy plots,
        group size=1 by 2,
        xticklabels at=edge bottom,
        vertical sep=0pt
    },
    width=0.9\linewidth,
]

\nextgroupplot[ymin=85,ymax=95,
               ytick={87.5, 90, ..., 95},
               xtick=\empty,
               ytick pos=left,
               height=1.5in,
               ylabel={\color{blue}{mCE} $\downarrow$ (\%)},
               legend image post style={scale=0},
               legend style={draw=none}
]
\addplot[mark=none] coordinates {(0, 0)};
\addlegendentry{\scriptsize TIN};
\addplot[mark=*,blue] coordinates {
      (15, 87.46)
      (12, 87.58)
      (10, 87.57)
      (7, 89.65)
      (5, 90.39)
      (3, 90.91)
      (1, 92.24)
    };

\nextgroupplot[ymin=86.75,ymax=88.75,
               axis x line=bottom,
               ytick={87, 87.5, 88, 88.5},
               xtick pos=lower,
               ytick pos=left,
               height=1.5in,
               xlabel={$1 / \lambda$},
               xmin=0,
               xmax=16,
               xtick={0, 2, 4, ..., 15},
               ylabel={\color{blue}{RA} $\uparrow$ (\%)},
               legend image post style={scale=0},
               legend style={draw=none},
               legend style={at={(0,1)},anchor=north west}
]
\addplot[mark=none] coordinates {(0, 0)};
\addlegendentry{\scriptsize C10};

\addplot[mark=*,blue] coordinates {
      (15, 88.31)
      (12, 88.07)
      (10, 88.21)
      (7, 87.40)
      (5, 87.24)
      (3, 87.06)
      (1, 87.00)
    };

\end{groupplot}

\begin{groupplot}[
    group style={
        group name=my fancy plots,
        group size=1 by 2,
        vertical sep=0pt,
    },
    width=0.9\linewidth,
]

\nextgroupplot[ymin=60,ymax=65,
               ytick={61, 63, ..., 65},
               xtick=\empty,
               ytick pos=right,
               height=1.5in
               ]
\addplot[mark=*,red,dashed] coordinates {
      (15, 62.05)
      (12, 62.51)
      (10, 62.52)
      (7, 62.16)
      (5, 62.28)
      (3, 62.59)
      (1, 62.67)
    };

\nextgroupplot[ymin=94,ymax=96.2,
               axis x line=bottom,
               ytick={94, 95, ..., 96},
               xtick=\empty,
               ytick pos=right,
               height=1.5in,
               xmin=0,
               xmax=16,
               ylabel={\color{red}{SA} $\uparrow$ (\%)},
               every axis y label/.append style={at=(ticklabel cs:1.1)}
]
\addplot[mark=*,red,dashed] coordinates {
      (15, 94.78)
      (12, 94.78)
      (10, 94.69)
      (7, 95.08)
      (5, 94.95)
      (3, 95.01)
      (1, 95.28)
    };
\end{groupplot}
\end{tikzpicture}
\vspace{-\baselineskip}
\caption{Trend of the mCE and SA with respect to the rate parameter. The models were trained using AFA in the auxiliary setting and no other augmentations for the main.}
\vspace{-\baselineskip}
\label{fig:strength_ablation}
\end{figure}
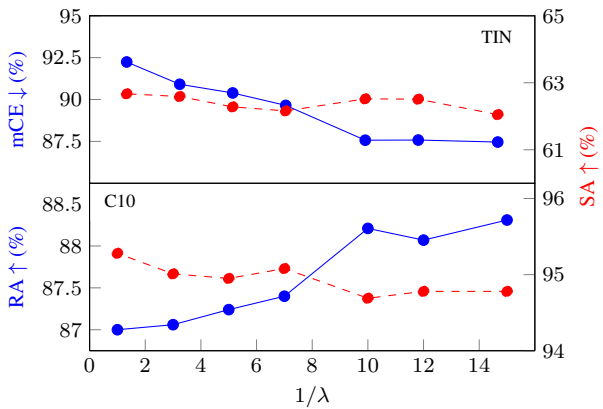

%% file: section/tables/tfba_augmax_comp.tex
\begin{table}[!b]
    \centering
    \footnotesize
    \begin{tabular}{ccccccc}
         \bfseries - & \bfseries Main & \bfseries Auxiliary & \bfseries SA$\uparrow$ & \bfseries RA$\uparrow$ & \bfseries mCE$\downarrow$ \\ \toprule
         \multirow{4}{*}{\rotatebox{90}{C10}} 
             & AugMix$^\dagger$   & \xmark & 95.47 & 86.48          & -          \\
             & AugMix$^\dagger$   & AugMax & 95.76 & \textbf{90.36} & - \\
             & AugMix$^\dagger$   & AFA    & 95.24 & 89.96          & -          \\ 
             & AugMix   & AFA    & 95.44 & 89.81                    & - \\
             \cmidrule{2-6}
         \multirow{4}{*}{\rotatebox{90}{C100}} 
             & AugMix$^\dagger$   & \xmark & 78.72 & 61.61          & -          \\
             & AugMix$^\dagger$   & AugMax & 78.69 & 65.75          & -          \\
             & AugMix$^\dagger$   & AFA   & 78.99 & 65.96          & - \\ 
             & AugMix   & AFA    & 77.80 & \textbf{66.69} & - \\
             \cmidrule{2-6}
         \multirow{4}{*}{\rotatebox{90}{TIN}} 
             & AugMix$^\dagger$   & \xmark & 64.65 & 36.30 & 83.90          \\
             & AugMix$^\dagger$   & AugMax & 62.21 & \textbf{38.67} & \textbf{80.72} \\
             & AugMix$^\dagger$   & AFA   & 64.34 & 38.53 & \textbf{80.79} \\ 
             & AugMix   & AFA    & 62.51 & \textbf{38.67} & 80.83 \\
             \cmidrule{2-6}
         \multirow{4}{*}{\rotatebox{90}{IN}} 
             & AugMix$^\dagger$   & \xmark & 65.2 & 31.5 & 87.1          \\
             & AugMix$^\dagger$   & AugMax & 66.5 & 36.5 & 80.6          \\
             & AugMix$^\dagger$   & AFA   & 65.0 & 36.8 & 80.4 \\ 
             & AugMix             & AFA   &  68.1 & \textbf{41.1} & \textbf{75.0}          \\
             \bottomrule
    \end{tabular}
    \caption{Comparison of AFA and AugMax (with AugMix for visual augmentation~\cite{Wang2021Oct}), with a ResNet18 backbone. The mark $^\dagger$ indicates the use of the JSD loss, otherwise the ACE loss is used.}
    \label{tab:tfba_augmax_comp}
\end{table}

%% file: section/tables/in/resnet.tex
\definecolor{darkgreen}{rgb}{0, 0.5, 0}
\newcommand{\gn}[1]{\textcolor{darkgreen}{#1}}
\newcommand{\rn}[1]{\textcolor{red}{#1}}

\begin{table*}[!t]
    \centering
    \scriptsize
    \begin{tabular}{cccccccccccHccl}
        & & & & \multicolumn{6}{c}{Robustness} & \multicolumn{3}{c}{Generalisation} & \multicolumn{2}{c}{Consistency} \\ \cmidrule(lr){5-10}\cmidrule(lr){11-13}\cmidrule(lr){14-15}
        & & & & \multicolumn{2}{c}{\bfseries IN-C} & \multicolumn{2}{c}{\bfseries IN-$\bar{\textbf{C}}$}& \multicolumn{2}{c}{\bfseries IN-3DCC} & \multicolumn{2}{c}{\bfseries IN-R} & \multicolumn{1}{c}{\bfseries IN-v2} & \multicolumn{2}{c}{\bfseries IN-P} \\
        & \bfseries Main & \bfseries Aux & \bfseries \textcolor{gray}{SA ($\uparrow$)} & \bfseries \textcolor{gray}{RA ($\uparrow$)} & \bfseries \textcolor{gray}{mCE ($\downarrow$)} & \bfseries \textcolor{gray}{RA ($\uparrow$)} & \bfseries \textcolor{gray}{mCE ($\downarrow$)} & \bfseries \textcolor{gray}{RA ($\uparrow$)} & \bfseries \textcolor{gray}{mCE ($\downarrow$)} & \bfseries \textcolor{gray}{Acc. ($\uparrow$)} 
        & \bfseries \textcolor{gray}{RMS ($\downarrow$)} 
        & \bfseries \textcolor{gray}{Avg. Acc. ($\uparrow$)} & \bfseries \textcolor{gray}{mFP ($\downarrow$)} & \bfseries \textcolor{gray}{mT5D ($\downarrow$)} \\ \toprule
        \multirow{11}{*}{\rotatebox{90}{ResNet18}}
        & -      & \xmark & \textbf{68.9} & 32.9 & 84.7 & 34.8 & 87.0 & 34.9 & 84.4 & 33.1 & 20.7 & \textbf{64.3} & 72.8 & 87.0 \\
        & -      & AFA   & \rn{68.2} & \gn{35.9} & \gn{81.0} & \gn{41.7} & \gn{78.3} & \gn{37.1} & \gn{81.7} & \rn{32.8} & 16.8 & \rn{63.7} & \gn{64.2} & \gn{76.8} \\
        \cmidrule{2-15}
        & AugMix$^\dagger$ & \xmark & 65.2 & 31.5 & 87.1 & 34.6 & 87.3 & 32.1 & 88.3 & 28.2 & 17.4 & 59.5 & 80.2 & 86.2 \\
        & AugMix$^\dagger$ & AFA   & \rn{65.0} & \gn{36.8} & \gn{80.4} & \gn{40.9} & \gn{79.3} & \gn{36.0} & \gn{83.2} & \gn{30.6} & 18.3 & \gn{60.9} & \gn{60.1} & \gn{68.5} \\
        & AugMix           & AFA   & \gn{68.1} & \gn{41.1} & \gn{75.0} & \gn{45.2} & \gn{73.3} & \gn{38.9} & \gn{79.4} & \gn{35.2} & \gn{14.6} & \gn{63.2} & \gn{68.5} & \gn{81.7} \\
        \cmidrule{2-15}
        & PRIME  & \xmark & 66.0 & 43.6 & 72.0 & 42.0 & 78.1 & 42.4 & 75.2 & 36.9 & 14.1 & 61.4 & 54.7 & 65.3 \\
        & PRIME  & AFA   & \gn{67.2} & \gn{\textbf{47.2}} & \gn{\textbf{67.8}} & \gn{\textbf{47.3}} & \gn{\textbf{71.1}} & \gn{\textbf{43.8}} & \gn{\textbf{73.5}} & \gn{\textbf{37.8}} & 13.1 & \gn{63.0} & \gn{\textbf{52.3}} & \gn{\textbf{63.7}} \\
        \cmidrule{2-15}
        & TA$^+$    & \xmark & \textbf{68.9} & 36.9 & 80.1 & 35.9 & 85.6 & 38.6 & 79.7 & 32.6 & 16.8 & 63.7 & 68.1 & 81.4 \\
        & TA$^+$    & AFA    & \rn{67.8} & \gn{41.4} & \gn{74.7} & \gn{42.9} & \gn{76.7} & \gn{41.1} & \gn{76.5} & \gn{35.4} & 14.2 & \rn{62.7} & \gn{59.9} & \gn{72.3} \\
        \midrule
        \multirow{11}{*}{\rotatebox{90}{ResNet50}}
        & -      & \xmark & 75.6 & 39.2 & 76.7 & 39.9 & 79.4 & 41.2 & 76.1 & 36.2 & 19.6 & 70.8 & 58.0 & 78.4 \\
        & -      & AFA   & \gn{76.5} & \gn{46.2} & \gn{68.0} & \gn{47.6} & \gn{69.4} & \gn{46.2} & \gn{69.8} & \gn{38.1} & 16.3 & \gn{72.0} & \gn{48.0} & \gn{67.2} \\ \cmidrule{2-15}
        & APR-SP & \xmark & 71.9 & 42.9 & 72.7 & 45.9 & 72.5 & 39.8 & 78.4 & 34.9 & 17.2 & 67.2 & 60.2 & 75.4 \\
        & APR-SP  & AFA   & \gn{74.4} & \gn{47.6} & \gn{66.7} & \gn{51.4} & \gn{64.9} & \gn{42.6} & \gn{74.6} & \gn{38.7} & 14.2 & \gn{69.3} & \gn{54.9} & \gn{72.6} \\
        \cmidrule{2-15}
        & AugMix$^{\dagger}$ & \xmark & 74.7 & 43.4 & 72.0 & 44.6 & 73.3 & 41.9 & 75.5 & 33.0 & 16.8 & 70.0 & 60.9 & 72.5 \\
        & AugMix$^\dagger$     & AFA   & \gn{75.6} & \gn{50.6} & \gn{62.9} & \gn{51.8} & \gn{64.0} & \gn{47.6} & \gn{68.3} & \gn{36.3} & 14.4 & \gn{71.2} & \gn{44.5} & \gn{56.1} \\
        & AugMix           & AFA   & \gn{76.6} & \gn{49.1} & \gn{64.7} & \gn{52.5} & \gn{62.9} & \gn{46.3} & \gn{69.6} & \gn{41.0} & \gn{14.1} & \gn{71.8} & \gn{52.2} & \gn{72.2} \\ 
        \cmidrule{2-15}
        & PRIME  & \xmark & 72.1 & 49.2 & 64.9 & 46.4 & 71.5 & 47.2 & 68.8 & 38.5 & 11.6 & 67.8 & 45.4 & 58.1 \\
        & PRIME  & AFA   & \gn{74.5} & \gn{\textbf{53.9}} & \gn{\textbf{59.2}} & \gn{\textbf{54.2}} & \gn{\textbf{61.3}} & \gn{\textbf{50.2}} & \gn{\textbf{65.0}} & \gn{\textbf{40.9}} & 11.6 & \gn{69.8} & \gn{\textbf{40.4}} & \gn{\textbf{54.8}} \\
        \cmidrule{2-15}
        & TA$^+$    & \xmark & 75.9 & 43.4 & 71.7 & 41.8 & 77.1 & 44.7 & 71.6 & 37.1 & 16.5 & 70.3 & 51.9 & 70.4 \\
        & TA$^+$    & AFA    & \gn{76.6} & \gn{50.3} & \gn{63.1} & \gn{49.7} & \gn{66.7} & \gn{49.6} & \gn{65.4} & \gn{40.0} & 14.3 & \gn{\textbf{72.2}} & \gn{45.1} & \gn{64.5} \\
        \midrule
        \multirow{10}{*}{\rotatebox{90}{CCT}}
        & -      & \xmark & 76.4 & 43.9 & 70.7 & 50.3 & 65.6 & 43.4 & 73.2 & 35.6 & 11.5 & 71.2 & 48.3 & 72.9 \\
        & -      & AFA   & \gn{76.9} & \gn{51.9} & \gn{61.0} & \gn{58.5} & \gn{55.4} & \gn{50.7} & \gn{64.4} & \gn{39.0} & 17.6 & \gn{71.9} & \gn{38.4} & \gn{61.8} \\
        \cmidrule{2-15}
        & AugMix & \xmark & 76.1 & 47.3 & 66.8 & 52.2 & 63.1 & 45.3 & 71.0 & 37.9 & 18.6 & 70.7 & 49.3 & 72.8 \\
        & AugMix & AFA    & \gn{\textbf{77.4}} & \gn{56.5} & \gn{55.6} & \gn{60.8} & \gn{52.2} & \gn{51.8} & \gn{62.8} & \gn{41.0} & \gn{17.4} & \gn{\textbf{72.5}} & \gn{37.9} & \gn{59.9} \\
        \cmidrule{2-15}
        & PRIME  & \xmark & 73.6 & 54.1 & 58.6 & 54.5 & 60.8 & 50.7 & 64.4 & 39.2 & 15.2 & 68.7 & 36.1 & 53.0 \\
        & PRIME  & AFA   & \gn{76.6} & \gn{\textbf{58.7}} & \gn{\textbf{52.8}} & \gn{\textbf{61.2}} & \gn{\textbf{52.0}} & \gn{\textbf{54.5}} & \gn{\textbf{59.4}} & \gn{\textbf{43.2}} & 15.8 & \gn{71.9} & \gn{\textbf{31.9}} & \gn{\textbf{51.2}} \\
        \cmidrule{2-15}
        & TA$^+$    & \xmark & 77.1 & 50.2 & 63.2 & 54.1 & 60.7 & 49.3 & 65.8 & 38.2 & 17.8 & 72.1 & 41.8 & 66.3 \\
        & TA$^+$    & AFA    & \rn{76.9} & \gn{56.0} & \gn{56.0} & \gn{59.1} & \gn{54.6} & \gn{53.1} & \gn{61.1} & \gn{41.1} & 15.0 & 72.1 & \gn{36.4} & \gn{58.5} \\
        \bottomrule
    \end{tabular}
    \caption{Robustness, generalization and consistency results on ImageNet-based benchmarks. Models with $^\dagger$ use the JSD loss. TrivialAugment (TA) has overlapping augmentations with IN-C ($^+$), and no other overlaps with other datasets. The \gn{green} colour indicates an improvement when the main augmentation is combined with AFA, while \rn{red} indicates no improvement. Results marked with \textbf{bold}/\textbf{\gn{bold}} are the best for a particular architecture.}
    \label{tab:imagenet_exps}
    \vspace{-\baselineskip}
\end{table*}

%% file: section/tables/c10/resnet18_cifar10.tex
\begin{table}[!t]
    \centering
    \footnotesize
    \begin{tabular}{ccccccc}
        &        &        & \multicolumn{2}{c}{\bfseries C10-C} & \multicolumn{2}{c}{\bfseries C100-C} \\
        \bfseries - & \bfseries Main & \bfseries Auxiliary & \bfseries SA$\uparrow$ & \bfseries RA$\uparrow$ & \bfseries SA$\uparrow$ & \bfseries RA$\uparrow$ \\ \toprule
        \multirow{6}{*}{\rotatebox{90}{ResNet18}}
        & -      & \xmark & 94.15 & 73.67 & 78.27 & 48.30 \\
        & -      & AFA   & \gn{94.69} & \gn{88.22} & \rn{77.91} & \gn{62.53} \\ \cmidrule{2-7}
        & AugMix$^\dagger$ & \xmark & \textbf{95.47} & 86.48 & 78.72 & 61.61 \\
        & AugMix$^\dagger$ & AFA   & \rn{95.24} & \gn{89.96} & \gn{\textbf{78.99}} & \gn{65.96} \\ \cmidrule{2-7}
        & PRIME  & \xmark & 94.38 & 89.81 & 75.49 & 66.16 \\
        & PRIME  & AFA   & \gn{94.54} & \gn{\textbf{90.64}} & \gn{76.16} & \gn{\textbf{68.48}} \\
        \midrule
        \multirow{6}{*}{\rotatebox{90}{CCT}}
        & -      & \xmark & 95.67 & 80.45 & \textbf{78.37} & 54.20 \\
        & -      & AFA   & \gn{\textbf{95.94}} & \gn{88.13} & \rn{77.47} & \gn{61.40} \\ \cmidrule{2-7}
        & AugMix & \xmark & 95.10 & 85.42 & 75.79 & 60.83 \\
        & AugMix & AFA   & \gn{\textbf{95.93}} & \gn{90.57} & \gn{77.22} & \gn{66.18} \\ \cmidrule{2-7}
        & PRIME  & \xmark & 95.30 & 90.56 & 76.65 & \textbf{67.92} \\
        & PRIME  & AFA   & \gn{95.49} & \gn{\textbf{91.40}} & \rn{76.50} & \rn{67.89} \\ 
        \midrule
                \multirow{2}{*}{\rotatebox{90}{CVT}}
        & -      & \xmark & 94.31 & 77.02 & 75.53 & 48.25 \\
        & -      & AFA   & \gn{\textbf{94.53}} & \gn{\textbf{87.03}} & \gn{\textbf{76.96}} & \gn{\textbf{60.12}} \\
        \midrule
        \multirow{2}{*}{\rotatebox{90}{VIT}}
        & -      & \xmark & 94.46 & 75.97 & 74.26 & 50.88 \\
        & -      & AFA   & \gn{\textbf{94.58}} & \gn{\textbf{86.71}} & \gn{\textbf{75.13}} & \gn{\textbf{58.25}} \\ 
        \bottomrule
    \end{tabular}
    \caption{Results for C10-C and C100-C with ResNet18, CCT. CVT and ViT-Lite. Models with $^\dagger$ use loss with JSD.}
    \vspace{-\baselineskip}
    \label{tab:low_res_exps}
\end{table}

%% file: section/tables/importance_of_aux.tex
\begin{table}[!t]
    \centering
    \footnotesize
    \begin{tabular}{ccccccc}
         \bfseries - & \bfseries Main & \bfseries Auxiliary & \bfseries SA$\uparrow$ & \bfseries RA$\uparrow$ & \bfseries mCE$\downarrow$ \\ \toprule
         \multirow{3}{*}{\rotatebox{90}{C10}}
             &   -  & \xmark & 94.15 & 73.67 & - \\
             & AFA & \xmark & 92.36 & 83.25 & - \\
             &   -  & AFA   & \bfseries 94.69 & \textbf{88.22} & - \\ \cmidrule{2-6}
         \multirow{3}{*}{\rotatebox{90}{C100}}
             &   -  & \xmark & 78.27 & 48.30 & - \\
             & AFA & \xmark & 72.34 & 58.70 & - \\
             &   -  & AFA   & \bfseries 77.91 & \textbf{62.53} & - \\ \cmidrule{2-6}
         \multirow{3}{*}{\rotatebox{90}{TIN}}
             &   -  & \xmark & 61.64 & 23.91 & 100.00 \\
             & AFA & \xmark & 59.04 & 28.87 &  93.45 \\
             &   -  & AFA   & \bfseries 62.52 & \textbf{33.35} &  \textbf{87.58} \\ \cmidrule{2-6}
         \multirow{3}{*}{\rotatebox{90}{IN}}
             &   -  & \xmark & 68.9 & 32.9 & 84.7 \\
             & AFA & \xmark & 66.7 & 33.3 & 84.4 \\
             &   -  & AFA   & \bfseries 68.2 & \textbf{35.9} & \textbf{81.0} \\ \bottomrule
    \end{tabular}
    \caption{Ablation results ResNet18 trained with and without Auxiliary Components on C10, C100, TinyImageNet and ImageNet.}
    \label{tab:afa_aux_ablation_comp}
    \vspace{-2\baselineskip}
\end{table}

%% file: section/conclusion.tex
\section{Conclusions}\label{sec:conclusion}
We proposed an efficient data augmentation technique called AFA, which complements existing visual augmentation techniques by filling the augmentation gap, that they do not cover in the Fourier domain. AFA perturbs the frequency components of images and generates adversarial samples.
By leveraging Fourier-basis functions and the auxiliary augmentation setting 
we demonstrate that AFA allows the models to learn from aggressive/adversarial input changes. We performed extensive experiments on benchmark datasets, and demonstrated that AFA benefits the robustness of models against common image corruptions, the consistency of predictions when facing increasing perturbations, and the OOD generalization performance. 
Being complementary to other augmentation techniques, AFA can further boost the robustness of models, especially against strong corruptions and perturbation, and it also results in better robustness in the frequency spectrum. 
We foresee that investigating the use of Fourier-basis functions on the training process of neural networks would provide promising improvement to model performance, thus encouraging their reliability in real scenarios.

%% file: appendix.tex
\section{Implementation Details}\label{app:impl_details}
Below, we report the training setup in detail. For all methods, and a particular dataset and architecture, the same training setup was used unless stated otherwise. 

\paragraph{Convolution Neural Networks}
For CIFAR-100 and Tiny ImageNet we use the SGD optimiser with an initial learning rate of 0.2, Nesterov momentum of 0.9 with a batch size of 128 training for 100 epochs. We use a weight decay of $0.0005$ and we do not decay the affine parameters of normalisation. 
For CIFAR-10, we follow the same setup as above, except we train for 200 epochs with a batch size of 256 and an initial learning rate of 0.1. The learning rate is decayed with a cosine annealing schedule to 0 which is stepped step-wise. For all models, we always employ the standard transformation of random crop with a padding of 4 and random horizontal flip.

For ImageNet, we follow~\cite{Hendrycks2019Dec} in that we use SGD optimiser with an initial learning rate of $0.1$ and Nesterov momentum of 0.9 and train for 90 epochs. We use a weight decay of $0.0001$ and we do not decay affine parameters of normalisation. The learning rate decays with a by a factor of $0.1$ every $30$ epochs. For all models, we  employ the standard transformation of random resized crop to image size of $224\times224$ with bilinear interpolation and random horizontal flip, before other augmentations.

We choose to train all models from scratch (no fine-tuning using AFA) so that we can study the effects of AFA without other underlying factors. Therefore, for fair comparison, we retrain PRIME from scratch as well using our setup.
For models trained with JSD, we follow~\cite{Wang2021Oct} for the regularising coefficient, mainly: $\lambda=10$ for CIFAR-10 and Tiny ImageNet, $\lambda=1$ for CIFAR-100 and $\lambda=12$ for ImageNet.

We only use the main BN layers during testing, similarly to AugMax for all convolution models.

\paragraph{Compact Convolution Transformer}
For CIFAR-10/100 and ImageNet we also train a transformer architecture. For all datasets we use CutMix (alpha=1.0) and MixUp (alpha=0.2 for ImageNet and alpha=1.0 for CIFAR-10/100) with an equal chance of applying one of the two. For CIFAR-10/100, we follow~\cite{Hassani2021Apr}. We train using the AdamW optimiser with max learning rate of $0.0006$ and weight decay of $0.06$, and we do not decay the affine parameters of the normalisation modules. We train with an effective batch size of 256, and apply learning rate decay following a cosine decay with a warm-up period of 10 epochs and the learning rate scheduler is stepped step-wise. For ImageNet, we use a max learning rate of $0.0005$, effective batch size of 1024 and a weight decay of $0.05$. The learning rate decay follows a cosine annealing schedule with a warm-up of $25$ epochs. The same standard transformations as for convolutional neural networks were applied.

\section{Evaluation metrics}\label{app:expl_metrics}
\paragraph{Mean corruption error (mCE)} measures the robustness of models against image corruptions~\cite{hendrycks2019benchmarking}, computed as:  
\begin{equation}
    \label{equ:1}
    \mathrm{mCE} = \frac{1}{|C|}\sum_{c \in C }\frac{\sum_{s=1}^5 E_{s,c}^f}{\sum_{s=1}^5 E_{s,c}^{baseline}},
\end{equation}
\noindent where the sum of classification error $E$ of five severity $s \in \{1,2,3,4,5\}$ per corruption $c$ of model $f$ is normalized by that of a baseline model. The normalized classification errors of all corruptions $C$ in the dataset are averaged to obtain mCE. We use AlexNet as baseline in ImageNet experiments and ResNet-18 for Tiny ImageNet. For CIFAR-10/100 there are no baselines advised so we do not report the mCE for these datasets.

\paragraph{Mean flip rate (mFR)} evaluates the consistency of model predictions with increasing perturbations~\cite{hendrycks2019benchmarking},  computed as follows:
\begin{equation}
\mathrm{mFR} = \frac{1}{|C|}\sum_{c\in C}\mathrm{FR}_c^f  = \frac{1}{|C|}\sum_{c\in C} \frac{\mathrm{FP}^f_c}{\mathrm{FP}^{baseline}_c},
\end{equation}
with
\begin{equation}
        \mathrm{FP}^f_c = \frac{1}{m(n-1)}\sum_{i=1}^m \sum_{j=2}^n \mathds{1}(f(x_j^{(i)})\neq f(x_{j-1}^{(i)})).
\end{equation}
\noindent 

$\mathds{1}(f(x_j^{(i)})\neq f(x_{j-1}^{(i)}))$ measures whether the prediction of the model $f$  on a frame $x_j$  is the same as its previous perturbed frame in the $i^{th}$ sequence.  If the predictions are the same, $\mathds{1}(f(x_j^{(i)})\neq f(x_{j-1}^{(i)}))$ equals to zero, and thus the performance of the model is not affected by the considered perturbations. $\mathrm{FP}^f_c$ measures the consistency of predictions over $m$  perturbed sequences, each with $n$ of frames. 
For a sequence corrupted by noise, the predictions are compared with those of the first frame, as noise is not temporally related. 
The $\mathrm{mFR}$ is obtained by averaging the normalized $\mathrm{FP}^f_c$ by that of a baseline model across all the perturbations $C$. The value of $\mathrm{mFR}$ is expected to be close to zero for a robust model.

\paragraph{Mean top-5 distance (mT5D)}  also measures the consistency of model predictions in terms of increasing perturbations~\cite{hendrycks2019benchmarking}. For a robust model, the top-5 predictions of frames over a sequence should be relevant to those of the previous frames in the sequence. The top-5 distance thus measures the inconsistency of top-5 predictions under consecutive perturbations, computed as follows:
\begin{equation}
    \mathrm{T5D}_c^{f} = \frac{1}{m(n-1)}\sum_{i=1}^m\sum_{j=2}^n d(\tau(x_j),\tau(x_{j-1})), 
\end{equation}
with
\begin{equation}
    d(\tau(x_j),\tau(x_{j-1})) = \sum_{i=1}^5 \sum_{j=min\{i,\rho(i)\}+1}^{max\{i,\rho(i)\}}\mathds{1}(1 \leq j-1 \leq 5),
\end{equation}
where $\rho(\tau(x_{j})(k)) =  \tau(x_{j-1})(k),$ $\tau(x_j)$ is the ranking of predictions for a perturbed frame $x_j$ and $\tau(x_j)(k)$ indicates the rank of the prediction being $k$. If $\tau(x_j)$ and $\tau(x_{j-1})$ are the same, then $d(\tau(x_j),\tau(x_{j-1})) = 0$.
Averaging the normalized $\mathrm{T5D}$  by that of the baseline over all corruptions  obtain  $\mathrm{mT5D} =  \frac{1}{|C|}\sum_{c\in C} \frac{\mathrm{T5D}_c^{f}}{\mathrm{T5D}_c^{baseline}}$.

\paragraph{Fourier heatmap} evaluates model robustness from a Fourier perspective~\cite{Yin2019Jun} exploiting Fourier basis functions to perturb test images and measuring the classification error of models. They are constructed as follows.
Let $U_{i,j} \in \mathds{R}^{d_1\times d_2}$ be a real-valued matrix such that its norm equals to 1. The Fourier transform of $U_{i,j}$  has only two non-zero elements located at $(i,j)$ and the corresponding symmetric coordinate with respect to the image center. Given an image $X$, a perturbed image with Fourier basis noise can be generated by $\Tilde{X}_{i,j}=X+rvU_{i,j}$, where $r$ is chosen randomly from a uniform distribution ranging from -1 to 1, and $v$ controls the strength of the added noise. Each channel of the images is perturbed independently with different $r$ and $v$. The model robustness against Fourier basis noise $U_{i,j}$ is evaluated by the classification error, and the final outcome is in a form of heatmap which records the error of the evaluated model under different Fourier basis noise. Examples are in~\cref{fig:cctheatmap}.

\section{Supplementary results}

\subsection{Results on Tiny ImageNet}\label{app:tin_results}

In Tab.~\ref{tab:tin_exps} we provide the robustness results on Tiny ImageNet (TIN), which are consistent with those presented on other datasets.
Models trained with AFA show robustness improvements consistently by significant margin with only negligible reduction of the clean accuracy.
We again see that JSD improves robustness slightly, and in AugMix it improves clean accuracy greatly.
\input{section/tables/tin/resnet18_tin}

\begin{figure}[!t]
    \centering
    \input{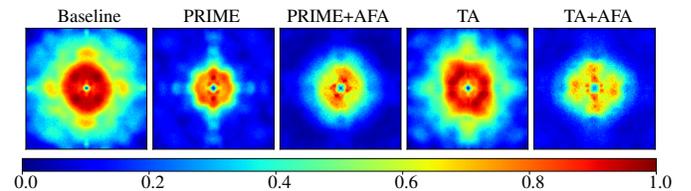}
    \caption{Fourier heatmaps of CCT trained with standard setting, PRIME, PRIME+AFA, TA and TA+AFA.}
    \label{fig:cctheatmap}
\end{figure}

\begin{figure*}[!t]
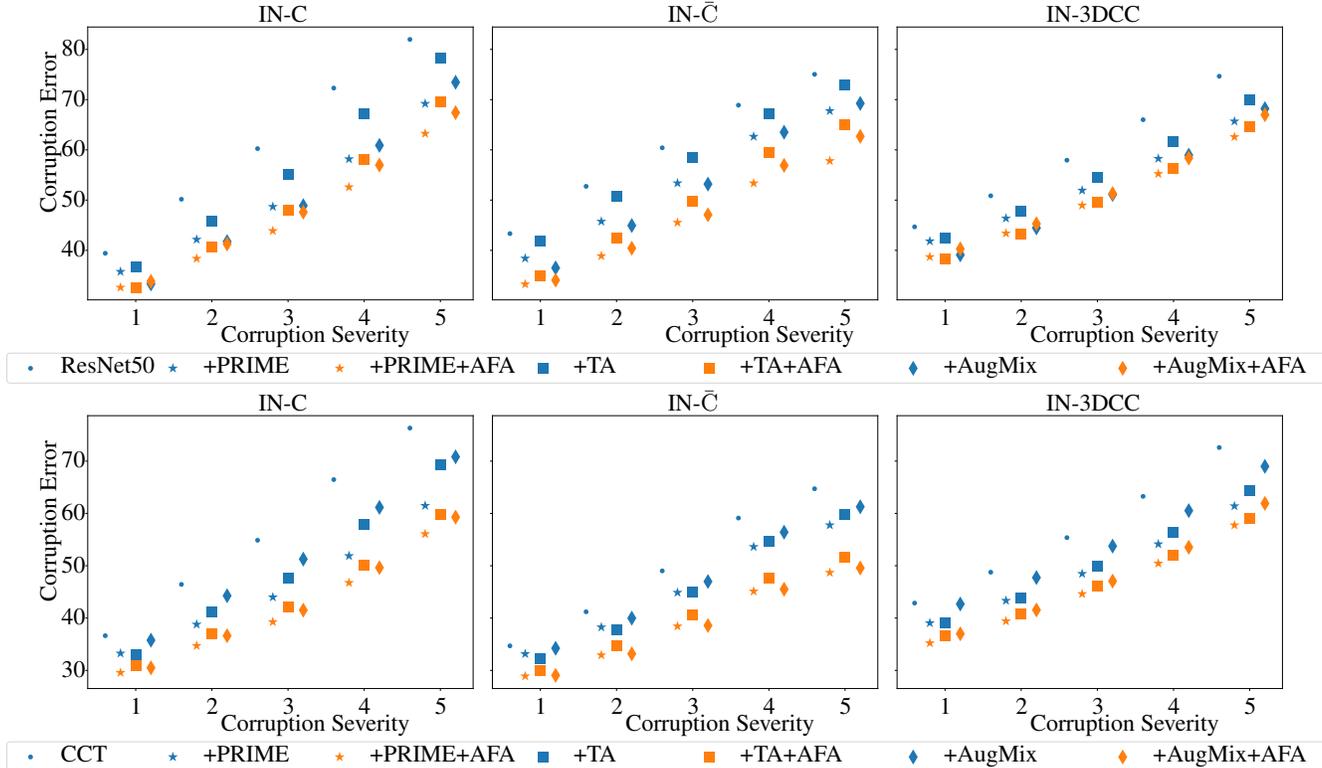

    \centering
    \input{section/figs/plot_benchmark_rn50}
    
    \input{section/figs/plot_benchmark_cct}
    \caption{Corruption error of ResNet50 and CCT trained with PRIME, PRIME+AFA, TA, TA+AFA, AugMix and AugMix+AFA. Models trained with AFA (orange points) have lower error at each severity than their counterpart trained with only visual augmentation (blue points), demonstrating the benefit of AFA to corruption robustness. }
    \label{fig:resultscorruptionseverity}
\end{figure*}

\subsection{Robustness in the frequency spectrum.}\label{app:heatmap_cct}
The Fourier heatmaps of CCT trained with standard setting, PRIME, PRIME+AFA, TA and TA+AFA are provided in~\cref{fig:cctheatmap}.
Our observations are consistent with those in the main paper. Also CCT models trained with the contribution of AFA have better robustness to low and middle-high frequency corruptions. 

\subsection{Robustness per corruption severity.}\label{app:per_severity}

We report the classification error of models tested under corruptions with different severity levels~\cref{fig:resultscorruptionseverity}. The models trained with AFA have consistently lower error than their counterpart trained without AFA, showing that AFA can further boost the robustness of models against common image corruptions, especially in difficult testing conditions with high severity.
 
\subsection{Robustness to each image corruption.}
\begin{figure*}
    \centering
    \input{section/figs/corrutpions/blur_noise}
    \input{section/figs/corrutpions/brown_noise}
    \input{section/figs/corrutpions/gaussian_noise}
    
    \vspace{-1.9em}
    \input{section/figs/corrutpions/impulse_noise}
    \input{section/figs/corrutpions/iso_noise}
    \input{section/figs/corrutpions/perlin_noise}

    \vspace{-1.9em}
    \input{section/figs/corrutpions/plasma_noise}
    \input{section/figs/corrutpions/shot_noise}    
    \input{section/figs/corrutpions/single_frequency_noise}

    \vspace{-1.9em}
    \input{section/figs/corrutpions/sine_waves}
    \input{section/figs/corrutpions/defocus}
    \input{section/figs/corrutpions/far_focus}

    \vspace{-1.9em}
    \input{section/figs/corrutpions/glass_blur}
    \input{section/figs/corrutpions/motion_blur}
    \input{section/figs/corrutpions/near_focus}

    \vspace{-1.9em}
     \input{section/figs/corrutpions/xy_motion}   
    \input{section/figs/corrutpions/z_motion}
    \input{section/figs/corrutpions/zoom_blur}

    \vspace{2.8em}
    \caption{Averaged classification error per corruption of ResNet50s (orange) and CCTs (green). The error points of  model trained with visual augmentations and additionally with AFA are connected. A decreasing line indicates better performance when trained additionally with AFA (a). }
    \label{fig:corruptions}
\end{figure*}

\begin{figure*}
    \centering
    \input{section/figs/corrutpions/sparkles}
    \input{section/figs/corrutpions/brightness}
    \input{section/figs/corrutpions/fog_3d}

    \vspace{-1.9em}
    \input{section/figs/corrutpions/fog}
    \input{section/figs/corrutpions/frost}
    \input{section/figs/corrutpions/low_light}

    \vspace{-1.9em}
   \input{section/figs/corrutpions/snow}
   \input{section/figs/corrutpions/biterror}
    \input{section/figs/corrutpions/caustic_refraction}
    
    \vspace{-1.9em}
    \input{section/figs/corrutpions/checkboard}
    \input{section/figs/corrutpions/color_quant}
    \input{section/figs/corrutpions/contrast}   

    \vspace{-1.9em}
    \input{section/figs/corrutpions/elastic_transform}
    \input{section/figs/corrutpions/flash}
    \input{section/figs/corrutpions/h265_abr}
    
    \vspace{-1.9em}
    \input{section/figs/corrutpions/inverse_sparkles}
    \input{section/figs/corrutpions/jpeg}
    \input{section/figs/corrutpions/pixelate}

    \vspace{2.8em}
    \caption{Averaged classification error per corruption of ResNet50s (orange) and CCTs (green). The error points of  model trained with visual augmentations and additionally with AFA are connected. A decreasing line indicates better performance when models are trained additionally with AFA (b). }
    \label{fig:corruptions}
\end{figure*}

Furthermore, we show the classification error averaged over five corruption severity levels per corruption type in~\cref{fig:corruptions}. The error points of model trained with visual augmentations only, and with further use of AFA are connected by a line. A downward trend means models trained with AFA have better robustness performance on specific corruption types. 
We observe that, in general, models with AFA have better corruption robustness than models trained only with visual augmentations. Significant improvements are especially evident on noise corruptions (Gaussian noise, impulse noise, iso noise, plasma noise, shot noise, single frequency grayscale noise and cocentric sine waves). One exception is ResNet50 trained with AugMix and AFA, for which the model trained without AFA performs better except on few cases. This can be attributed to the less training time (90 epoch vs 180 epochs) than that of ResNet50+AugMix.

\section{Evidence of adversarial nature of AFA}\label{app:adv_evidence}

\subsection{Main and auxiliary batch normalisation}\label{app:bn_diff}

For the ResNet architecture, which includes Batch Normalisation layers, we had replaced the Batch Normalisation layers with DuBIN layers~\cite{Wang2021Oct} while operating the Auxiliary setting. Assuming that there is no difference in the distribution of images augmented using AFA and a typical visual augmentation technique, there should be no difference in the affine parameters learnt for each individual batch normalisation parameter (the main and the auxiliary). 

We show in Fig.~\ref{fig:main_aux_diff} the Mean Absolute Difference of the same parameter between the main and the auxiliary component of the DuBIN layer at different depths of the model. We show the results for for models trained with ACE loss for ResNet-50 where AFA is paired with just standard transforms, AugMix, PRIME and Trivial Augment (TA).

We can see that at earlier depths the parameter differ largely, which is explained by the difference in distribution of a visually augmented and AFA augmented image. This difference converges to a lower value, which is again explained by the model attempting to extract similar features from the differently augmented images.

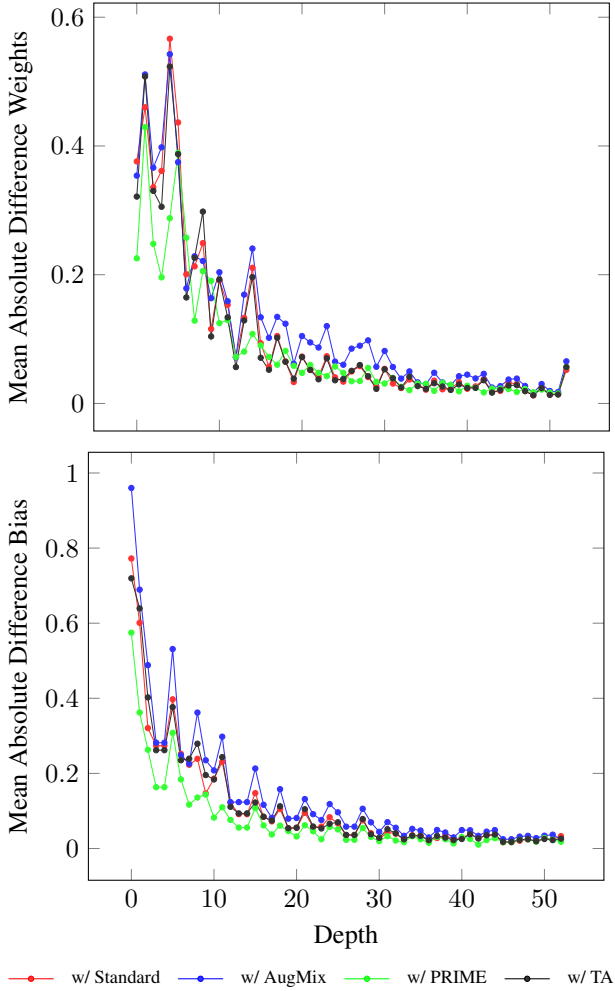
\begin{figure}[!t]
\centering
\begin{subfigure}[t]{\linewidth}

\begin{tikzpicture}
\pgfplotstableread[col sep=comma]{section/data/batchnorm_diffs.csv}\datatable

\begin{axis}[
ylabel={Mean Absolute Difference Weights},
xticklabels={,,}
]
 \addplot[red!80!white,mark=*,mark options={scale=0.5}] table [x index=0, y index=3] {\datatable};

 \addplot[blue!80!white,mark=*,mark options={scale=0.5}] table [x index=0, y index=1] {\datatable};

 \addplot[green!80!white,mark=*,mark options={scale=0.5}] table [x index=0, y index=5] {\datatable};

 \addplot[black!80!white,mark=*,mark options={scale=0.5}] table [x index=0, y index=7] {\datatable};
\end{axis}
\end{tikzpicture}

\begin{tikzpicture}
\pgfplotstableread[col sep=comma]{section/data/batchnorm_diffs.csv}\datatable

\begin{axis}[
ylabel={Mean Absolute Difference Bias},
xlabel={Depth},
]
 \addplot[red!80!white,mark=*,mark options={scale=0.5}] table [x index=0, y index=4] {\datatable};

 \addplot[blue!80!white,mark=*,mark options={scale=0.5}] table [x index=0, y index=2] {\datatable};

 \addplot[green!80!white,mark=*,mark options={scale=0.5}] table [x index=0, y index=6] {\datatable};

 \addplot[black!80!white,mark=*,mark options={scale=0.5}] table [x index=0, y index=8] {\datatable};
\end{axis}
\end{tikzpicture}
\end{subfigure}

\begin{subfigure}[b]{\linewidth}
    \centering
    
    \begin{tikzpicture}
        \begin{customlegend}[legend columns=-1,legend style={draw=none,column sep=1ex},legend entries={\scriptsize w/ Standard, \scriptsize w/ AugMix, \scriptsize w/ PRIME, \scriptsize w/ TA}]
        \addlegendimage{red,fill=red!80!white,mark=*,mark options={scale=0.5}, sharp plot}
        \addlegendimage{blue,fill=blue!80!white,mark=*,mark options={scale=0.5}, sharp plot}
        \addlegendimage{green,fill=green!80!white,mark=*,mark options={scale=0.5}, sharp plot}
        \addlegendimage{black,fill=black!80!white,mark=*,mark options={scale=0.5}, sharp plot}
        \end{customlegend}
    \end{tikzpicture}
\end{subfigure}
    \vspace{-2em}
    \caption{Comparison of the mean absolute difference of the learnt affine parameters for the two batch normalisations in the Dual Batch Norm Layers of ResNet50-DuBIN architecture at different depths.}
    \label{fig:main_aux_diff}
\end{figure}

\subsection{Embedding Space Visualization}\label{app:pca_vis}
We compare how diverse are the augmentations of AFA are with respect to other methods. 
We follow the procedure in ~\cite{Modas2021Dec}. 
To reiterate the procedure, we randomly select 3 images from ImageNet, each one belonging to a different class. 
For each image, we generate 100 transformed instances using Standard Transform, Trivial Augment, PRIME, PGD attack with the following parameters: 5 steps, epsilon of $8/255$ and alpha of $2/255$, and with AFA.
Then, we pass the transformed instances of each method through a ResNet-50 pre-trained on ImageNet using standard transform and training setup, and extract the features of its embedding space from the penultimate layer before the dense layer.
On the features extracted for each method, we perform
PCA after whitening and then visualize the projection of the features onto the first two principal components. 
We visualize the projected augmented space in Fig.~\ref{fig:emb_vis}, which
demonstrates that AFA generates which are more akin to an adversarial attack rather than a standard augmentation. This is clear from a visual similarity of AFA's result in Fig.~\ref{subfig:emb_vis_afa} to PGD's result in Fig.~\ref{subfig:emb_vis_pgd} and dissimilarity to the other Visual Augmentation techniques.

Finally, we also add in Fig.~\ref{subfig:emb_vis_afa_with_aux} the embedding space visualisation for the Auxiliary Trained model with AFA augmentation and standard transform for main, following the same procedure as above. We see that the model learns more separable embeddings for images augmented with AFA using the auxiliary setting, therefore is less sensitive to Frequency perturbation. The embeddings also retain a large variance and hardness, therefore showcasing the diversity of the augmentations of AFA.

\begin{figure*}[!t]
    \centering
    \begin{subfigure}[b]{0.33\linewidth}
\begin{tikzpicture}
    \pgfplotstableread[col sep=comma]{section/data/pca_components_std.csv}\datatable
\begin{axis}[width=\linewidth]
    \addplot[
        scatter/classes={0={blue!60!white}, 1={red!60!white}, 2={green!60!white}},
        scatter, mark=*, only marks, 
        scatter src=explicit symbolic,
        visualization depends on={value \thisrow{label} \as \Label} 
    ] table [x index = 1, y index = 2, meta=label] {\datatable};
\end{axis}
\end{tikzpicture}
\caption{Standard Transform}
    \end{subfigure}%
    \begin{subfigure}[b]{0.33\linewidth}
\begin{tikzpicture}
    \pgfplotstableread[col sep=comma]{section/data/pca_components_ta.csv}\datatable
\begin{axis}[width=\linewidth]
    \addplot[
        scatter/classes={0={blue!60!white}, 1={red!60!white}, 2={green!60!white}},
        scatter, mark=*, only marks, 
        scatter src=explicit symbolic,
        visualization depends on={value \thisrow{label} \as \Label} 
    ] table [x index = 1, y index = 2, meta=label] {\datatable};
\end{axis}
\end{tikzpicture}
\caption{Trivial Augment}
    \end{subfigure}%
    \begin{subfigure}[b]{0.33\linewidth}
\begin{tikzpicture}
    \pgfplotstableread[col sep=comma]{section/data/pca_components_prime.csv}\datatable
\begin{axis}[width=\linewidth]
    \addplot[
        scatter/classes={0={blue!60!white}, 1={red!60!white}, 2={green!60!white}},
        scatter, mark=*, only marks, 
        scatter src=explicit symbolic,
        visualization depends on={value \thisrow{label} \as \Label} 
    ] table [x index = 1, y index = 2, meta=label] {\datatable};
\end{axis}
\end{tikzpicture}
\caption{PRIME}
    \end{subfigure}

    \begin{subfigure}[b]{0.33\linewidth}
\begin{tikzpicture}
    \pgfplotstableread[col sep=comma]{section/data/pca_components_pgd.csv}\datatable
\begin{axis}[width=\linewidth]
    \addplot[
        scatter/classes={0={blue!60!white}, 1={red!60!white}, 2={green!60!white}},
        scatter, mark=*, only marks, 
        scatter src=explicit symbolic,
        visualization depends on={value \thisrow{label} \as \Label} 
    ] table [x index = 1, y index = 2, meta=label] {\datatable};
\end{axis}
\end{tikzpicture}
\caption{PGD Attack}
\label{subfig:emb_vis_pgd}
    \end{subfigure}%
    \begin{subfigure}[b]{0.33\linewidth}
\begin{tikzpicture}
    \pgfplotstableread[col sep=comma]{section/data/pca_components_fourier.csv}\datatable
\begin{axis}[width=\linewidth]
    \addplot[
        scatter/classes={0={blue!60!white}, 1={red!60!white}, 2={green!60!white}},
        scatter, mark=*, only marks, 
        scatter src=explicit symbolic,
        visualization depends on={value \thisrow{label} \as \Label} 
    ] table [x index = 1, y index = 2, meta=label] {\datatable};
\end{axis}
\end{tikzpicture}
\caption{AFA on Standard}
\label{subfig:emb_vis_afa}
    \end{subfigure}%
    \begin{subfigure}[b]{0.33\linewidth}
\begin{tikzpicture}
    \pgfplotstableread[col sep=comma]{section/data/afa_pca_components_afa.csv}\datatable
\begin{axis}[width=\linewidth]
    \addplot[
        scatter/classes={0={blue!60!white}, 1={red!60!white}, 2={green!60!white}},
        scatter, mark=*, only marks, 
        scatter src=explicit symbolic,
        visualization depends on={value \thisrow{label} \as \Label} 
    ] table [x index = 1, y index = 2, meta=label] {\datatable};
\end{axis}
\end{tikzpicture}
\caption{AFA on Auxiliary Trained}
\label{subfig:emb_vis_afa_with_aux}
    \end{subfigure}%
    \caption{Differences in the Embedding Space for Different Methods and PGD Attack. From (a)-(e) the standardly trained model is used, and for (f) the model trained in the auxiliary setting is used.}
    \label{fig:emb_vis}
\end{figure*}
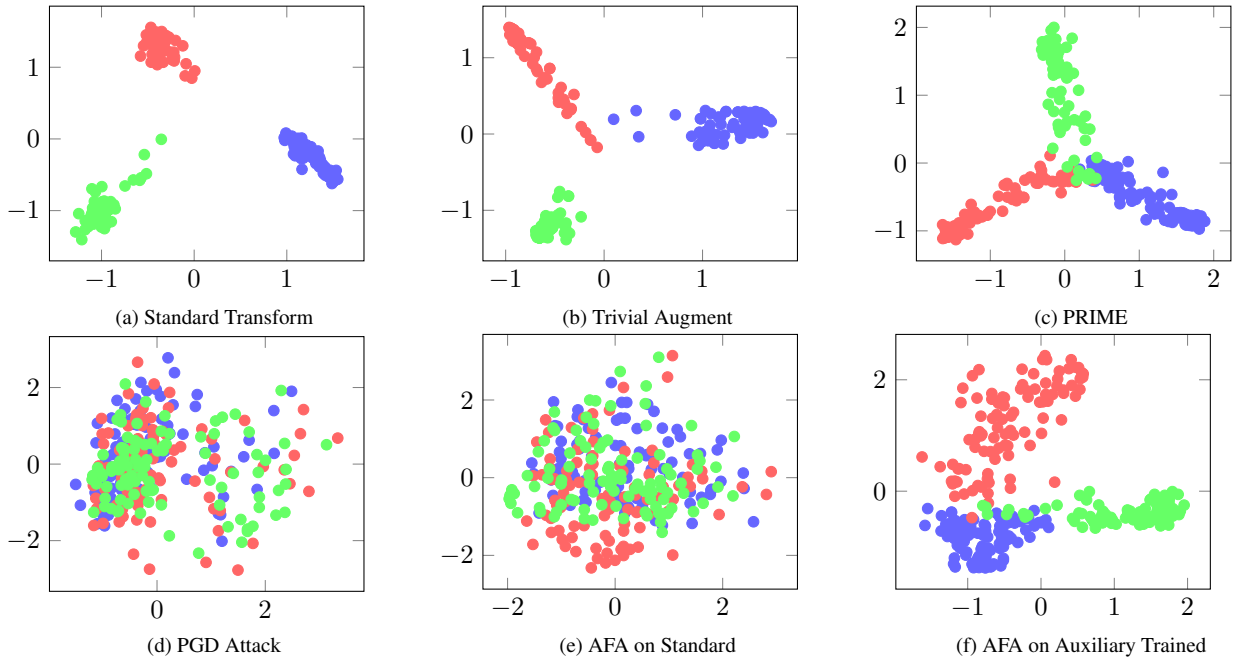

\section{Regularisation Effect}\label{app:reg_effect}
In Fig.~\ref{fig:reg_effect} we show the norm of the weights of the convolutional kernels for the ResNet50 models trained with and without AFA at each depth. We see that AFA provides a strong regularisation effect that is akin to the regularisation effect of PRIME. Meanwhile, we see that AugMix does not regularise the weights at all compared to the baseline model with only the standard transforms. The weights are however regularised to when AFA is paired with AugMix. Combined with PRIME, there does not seem to be further regularisation of the weights.

\begin{figure*}[!t]
\centering
\begin{subfigure}[b]{\linewidth}
    \begin{tikzpicture}
        \pgfplotstableread[col sep=comma]{section/data/conv2d_norms.csv}\datatable
\begin{axis}[
ylabel={Weight Norm of Conv2d Layers},
xlabel={Depth},
width=\linewidth,
height=7cm
]
 \addplot[red!80!white,mark=*,mark options={scale=0.5}] table [x index=0, y index=1] {\datatable};

 \addplot[blue!80!white,mark=*,mark options={scale=0.5}] table [x index=0, y index=2] {\datatable};

 \addplot[green!80!white,mark=*,mark options={scale=0.5}] table [x index=0, y index=3] {\datatable};

 \addplot[black!80!white,mark=*,mark options={scale=0.5}] table [x index=0, y index=4] {\datatable};

  \addplot[yellow!80!white,mark=*,mark options={scale=0.5}] table [x index=0, y index=5] {\datatable};

   \addplot[brown!80!white,mark=*,mark options={scale=0.5}] table [x index=0, y index=6] {\datatable};
\end{axis}
    \end{tikzpicture}
\end{subfigure}%

\begin{subfigure}[b]{\linewidth}
    \centering
    
    \begin{tikzpicture}
        \begin{customlegend}[legend columns=-1,legend style={draw=none,column sep=1ex},legend entries={\scriptsize Standard + AFA, \scriptsize Standard, \scriptsize PRIME + AFA, \scriptsize PRIME, \scriptsize AugMix + AFA, \scriptsize AugMix}]
        \addlegendimage{red,fill=red!80!white,mark=*,mark options={scale=0.5}, sharp plot}
        \addlegendimage{blue,fill=blue!80!white,mark=*,mark options={scale=0.5}, sharp plot}
        \addlegendimage{green,fill=green!80!white,mark=*,mark options={scale=0.5}, sharp plot}
        \addlegendimage{black,fill=black!80!white,mark=*,mark options={scale=0.5}, sharp plot}
        \addlegendimage{yellow,fill=yellow!80!white,mark=*,mark options={scale=0.5}, sharp plot}
        \addlegendimage{brown,fill=brown!80!white,mark=*,mark options={scale=0.5}, sharp plot}
        \end{customlegend}
    \end{tikzpicture}
\end{subfigure}
\caption{The norm of the Conv2d Layers for ResNet 50 trained with different augmentation techniques with and without AFA. The plot highlights the regularisation effect the methods have on the model weights.}
\label{fig:reg_effect}
\end{figure*}
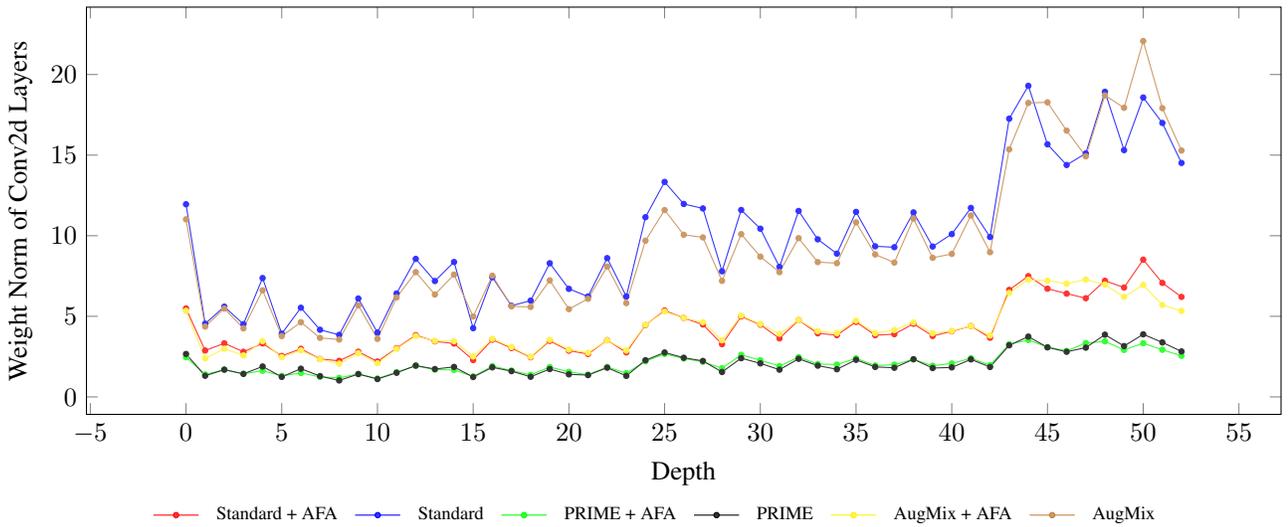

\section{Proof of Augmenting Fourier Domain}\label{app:proof}

\newtheorem{lemma}{Lemma}
\newtheorem{theorem}{Theorem}

\begin{lemma}[Linearity]\label{lem:linearity}
Let $f$, $g$ be functions of a real variable and let $\mathscr{F}(f)$ and $\mathscr{F}(g)$ be their Fourier transforms. Then for complex numbers $a$ and $b$
\begin{equation}
    \mathscr{F}(af+bg)=a\mathscr{F}(f)+b\mathscr{F}(g),
\end{equation}
therefore, Fourier transform $\mathscr{F}$ is a linear transformation.
\end{lemma}

\begin{lemma}[Fourier Transform of Plane Wave]
\label{lem:dirac}
The Fourier transform of the planar wave given by the frequency $f$ and the direction $\omega$, $A_{f, \omega}$ has a fourier transform
\begin{align}
    \mathscr{F}(A_{f, \omega}) 
    &= \mathscr{F}\left( R\cos(2\pi f(u\cos(\omega) + v\sin(\omega))) \right) \\
    &= \frac{R}{2} \left( \delta(\hat{x}, \hat{y}) + \delta(\bar{x}, \bar{y}) \right),
\end{align}
where, $\hat{x} = x - f \cos(\omega)$, $\hat{y} = y - f \sin(\omega)$ and $\bar{x} = x + f \cos(\omega)$, $\bar{y} = y + f \sin(\omega)$.
\end{lemma}

\begin{theorem}[AFA Augments the Fourier Domain]
Given an image sample $s$, an augmentation using AFA produces as augmentation in the Fourier domain of the image for one specific frequency and orientation of the wave $(f, \omega)$.
\end{theorem}

\begin{proof}
Given image $s$ and the randomly sampled planar wave using AFA, $\sigma A_{f, \omega}$, dropping the subscript for the channels for clarity, we have:
\begin{align}
    \mathscr{F}(\text{AFA}(s)) 
    &= \mathscr{F}(s + \sigma A_{f, \omega}) \nonumber \\
    &= \mathscr{F}(s) + \sigma \mathscr{F}(A_{f, \omega}) \\
    & \text{(using Lemma~\ref{lem:linearity})} \nonumber \\
    &= \mathscr{F}(s) + \frac{\sigma R}{2} \left( \delta(\hat{x}, \hat{y}) + \delta(\bar{x}, \bar{y}) \right). \\
    & \text{(using Lemma~\ref{lem:dirac})} \nonumber
\end{align}
Therefore, we prove augmenting an image $s$ with AFA corresponds to augmenting the amplitude of a specific frequency component ($f$,$\omega$) in the 2D Fourier transform of the image.
\end{proof}

%% file: section/tables/tin/resnet18_tin.tex
\begin{table}[!t]
	\centering
	\footnotesize
	\begin{tabular}{cccccc}
		&        &        &          & \multicolumn{2}{c}{\bfseries TIN-C} \\
		\bfseries - & \bfseries Main   & \bfseries Auxiliary & \bfseries SA$\uparrow$ & \bfseries RA$\uparrow$ & \bfseries mCE$\downarrow$ \\ \toprule
		\multirow{13}{*}{\rotatebox{90}{ResNet18}}
		            & -                & \xmark              & 63.56                  & 25.86                  & 97.34                     \\
		            & AFA              & \xmark              & 59.04                  & 28.87                  & 93.45                     \\
		            & -                & AFA                 & \rn{62.52}             & \gn{33.35}             & \gn{87.58}                \\
		\cmidrule{2-6}
		            & AugMix           & \xmark              & 62.95                  & 36.26                  & 84.05                     \\
		            & AugMix           & AFA                 & \rn{62.51}             & \gn{38.67}             & \gn{80.83}                \\
		            & AugMix$^\dagger$ & \xmark              & \textbf{64.65}                  & 36.30                  & 83.90                     \\
		            & AugMix$^\dagger$ & AFA                 & \rn{64.34}             & \gn{38.52}             & \gn{80.79}                \\
		\cmidrule{2-6}
		            & PRIME            & \xmark              & 63.07                  & 39.67                  & 79.42                     \\
		            & PRIME            & AFA                 & \rn{62.48}             & \gn{41.09}             & \gn{77.55}                \\
		            & PRIME$^\dagger$  & \xmark              & 63.24                  & 41.22                  & 77.44                     \\
		            & PRIME$^\dagger$  & AFA                 & \rn{62.65}             & \textbf{\gn{43.00}}             & \textbf{\gn{73.11}}                \\
		\bottomrule
	\end{tabular}
	\caption{Results for TIN-C with ResNet18. Models with $^\dagger$ use loss with JSD.}
	\label{tab:tin_exps}
\end{table}

%% file: main.bbl
\begin{thebibliography}{10}\itemsep=-1pt

\bibitem{chen2020simple}
Ting Chen, Simon Kornblith, Mohammad Norouzi, and Geoffrey Hinton.
\newblock A simple framework for contrastive learning of visual
  representations, 2020.

\bibitem{cubuk2019autoaugment}
Ekin~D. Cubuk, Barret Zoph, Dandelion Mane, Vijay Vasudevan, and Quoc~V. Le.
\newblock Autoaugment: Learning augmentation policies from data, 2019.

\bibitem{NEURIPS2020_d85b63ef}
Ekin~Dogus Cubuk, Barret Zoph, Jon Shlens, and Quoc Le.
\newblock Randaugment: Practical automated data augmentation with a reduced
  search space.
\newblock In H. Larochelle, M. Ranzato, R. Hadsell, M.F. Balcan, and H. Lin,
  editors, {\em Advances in Neural Information Processing Systems}, volume~33,
  pages 18613--18624. Curran Associates, Inc., 2020.

\bibitem{5206848}
Jia Deng, Wei Dong, Richard Socher, Li-Jia Li, Kai Li, and Li Fei-Fei.
\newblock Imagenet: A large-scale hierarchical image database.
\newblock In {\em 2009 IEEE Conference on Computer Vision and Pattern
  Recognition}, pages 248--255, 2009.

\bibitem{chen2021amplitudephase}
Chen \etal.
\newblock Amplitude-phase recombination: Rethinking robustness of convolutional
  neural networks in frequency domain, 2021.

\bibitem{hassani2022escaping}
Hassani \etal.
\newblock Escaping the big data paradigm with compact transformers, 2022.

\bibitem{faghri2023reinforce}
Fartash Faghri, Hadi Pouransari, Sachin Mehta, Mehrdad Farajtabar, Ali Farhadi,
  Mohammad Rastegari, and Oncel Tuzel.
\newblock Reinforce data, multiply impact: Improved model accuracy and
  robustness with dataset reinforcement, 2023.

\bibitem{Gao_2023_ICCV}
Zhiqiang Gao, Kaizhu Huang, Rui Zhang, Dawei Liu, and Jieming Ma.
\newblock Towards better robustness against common corruptions for unsupervised
  domain adaptation.
\newblock In {\em Proceedings of the IEEE/CVF International Conference on
  Computer Vision (ICCV)}, pages 18882--18893, October 2023.

\bibitem{greco2023}
Antonio Greco, Nicola Strisciuglio, Mario Vento, and Vincenzo Vigilante.
\newblock Benchmarking deep networks for facial emotion recognition in the
  wild.
\newblock {\em Multimedia Tools and Applications}, 82(8):11189--11220, 2023.

\bibitem{Hao_2023_WACV}
Xiaoshuai Hao, Yi Zhu, Srikar Appalaraju, Aston Zhang, Wanqian Zhang, Bo Li,
  and Mu Li.
\newblock Mixgen: A new multi-modal data augmentation.
\newblock In {\em Proceedings of the IEEE/CVF Winter Conference on Applications
  of Computer Vision (WACV) Workshops}, pages 379--389, January 2023.

\bibitem{Hassani2021Apr}
Ali Hassani, Steven Walton, Nikhil Shah, Abulikemu Abuduweili, Jiachen Li, and
  Humphrey Shi.
\newblock {Escaping the Big Data Paradigm with Compact Transformers}.
\newblock {\em arXiv}, Apr. 2021.

\bibitem{he2015deep}
Kaiming He, Xiangyu Zhang, Shaoqing Ren, and Jian Sun.
\newblock Deep residual learning for image recognition, 2015.

\bibitem{hendrycks2021faces}
Dan Hendrycks, Steven Basart, Norman Mu, Saurav Kadavath, Frank Wang, Evan
  Dorundo, Rahul Desai, Tyler Zhu, Samyak Parajuli, Mike Guo, Dawn Song, Jacob
  Steinhardt, and Justin Gilmer.
\newblock The many faces of robustness: A critical analysis of
  out-of-distribution generalization, 2021.

\bibitem{hendrycks2019benchmarking}
Dan Hendrycks and Thomas Dietterich.
\newblock Benchmarking neural network robustness to common corruptions and
  perturbations, 2019.

\bibitem{Hendrycks2019Dec}
Dan Hendrycks, Norman Mu, Ekin~D. Cubuk, Barret Zoph, Justin Gilmer, and Balaji
  Lakshminarayanan.
\newblock {AugMix: A Simple Data Processing Method to Improve Robustness and
  Uncertainty}.
\newblock {\em arXiv}, Dec. 2019.

\bibitem{pmlr-v202-hounie23a}
Ignacio Hounie, Luiz F.~O. Chamon, and Alejandro Ribeiro.
\newblock Automatic data augmentation via invariance-constrained learning.
\newblock In Andreas Krause, Emma Brunskill, Kyunghyun Cho, Barbara Engelhardt,
  Sivan Sabato, and Jonathan Scarlett, editors, {\em Proceedings of the 40th
  International Conference on Machine Learning}, volume 202 of {\em Proceedings
  of Machine Learning Research}, pages 13410--13433. PMLR, 23--29 Jul 2023.

\bibitem{Kamann2021}
Christoph Kamann and Carsten Rother.
\newblock Benchmarking the robustness of semantic segmentation models with
  respect to common corruptions.
\newblock {\em International Journal of Computer Vision}, 129(2):462--483, Feb
  2021.

\bibitem{kar20223d}
Oğuzhan~Fatih Kar, Teresa Yeo, Andrei Atanov, and Amir Zamir.
\newblock 3d common corruptions and data augmentation, 2022.

\bibitem{NEURIPS2020_2ba59664}
Ildoo Kim, Younghoon Kim, and Sungwoong Kim.
\newblock Learning loss for test-time augmentation.
\newblock In H. Larochelle, M. Ranzato, R. Hadsell, M.F. Balcan, and H. Lin,
  editors, {\em Advances in Neural Information Processing Systems}, volume~33,
  pages 4163--4174. Curran Associates, Inc., 2020.

\bibitem{cifar10}
Alex Krizhevsky, Vinod Nair, and Geoffrey Hinton.
\newblock Cifar-10 (canadian institute for advanced research).

\bibitem{cifar100}
Alex Krizhevsky, Vinod Nair, and Geoffrey Hinton.
\newblock Cifar-100 (canadian institute for advanced research).

\bibitem{tin}
Ya Le and Xuan~S. Yang.
\newblock Tiny imagenet visual recognition challenge.
\newblock 2015.

\bibitem{9412611}
Xiu-Chuan Li, Xu-Yao Zhang, Fei Yin, and Cheng-Lin Liu.
\newblock F-mixup: Attack cnns from fourier perspective.
\newblock In {\em 2020 25th International Conference on Pattern Recognition
  (ICPR)}, pages 541--548, 2021.

\bibitem{liu2023improving}
Chang Liu, Wenzhao Xiang, Yuan He, Hui Xue, Shibao Zheng, and Hang Su.
\newblock Improving model generalization by on-manifold adversarial
  augmentation in the frequency domain, 2023.

\bibitem{liu2023outofdistribution}
Jiashuo Liu, Zheyan Shen, Yue He, Xingxuan Zhang, Renzhe Xu, Han Yu, and Peng
  Cui.
\newblock Towards out-of-distribution generalization: A survey, 2023.

\bibitem{Liu_2023_ICCV}
Siao Liu, Zhaoyu Chen, Yang Liu, Yuzheng Wang, Dingkang Yang, Zhile Zhao,
  Ziqing Zhou, Xie Yi, Wei Li, Wenqiang Zhang, and Zhongxue Gan.
\newblock Improving generalization in visual reinforcement learning via
  conflict-aware gradient agreement augmentation.
\newblock In {\em Proceedings of the IEEE/CVF International Conference on
  Computer Vision (ICCV)}, pages 23436--23446, October 2023.

\bibitem{Liu_2023_CVPR}
Yang Liu, Shen Yan, Laura Leal-Taix\'e, James Hays, and Deva Ramanan.
\newblock Soft augmentation for image classification.
\newblock In {\em Proceedings of the IEEE/CVF Conference on Computer Vision and
  Pattern Recognition (CVPR)}, pages 16241--16250, June 2023.

\bibitem{101007}
Yuyang Long, Qilong Zhang, Boheng Zeng, Lianli Gao, Xianglong Liu, Jian Zhang,
  and Jingkuan Song.
\newblock Frequency domain model augmentation for adversarial attack.
\newblock In Shai Avidan, Gabriel Brostow, Moustapha Ciss{\'e}, Giovanni~Maria
  Farinella, and Tal Hassner, editors, {\em Computer Vision -- ECCV 2022},
  pages 549--566, Cham, 2022. Springer Nature Switzerland.

\bibitem{101007978}
Yuyang Long, Qilong Zhang, Boheng Zeng, Lianli Gao, Xianglong Liu, Jian Zhang,
  and Jingkuan Song.
\newblock Frequency domain model augmentation for adversarial attack.
\newblock In Shai Avidan, Gabriel Brostow, Moustapha Ciss{\'e}, Giovanni~Maria
  Farinella, and Tal Hassner, editors, {\em Computer Vision -- ECCV 2022},
  pages 549--566, Cham, 2022. Springer Nature Switzerland.

\bibitem{ma2023learning}
Guozheng Ma, Linrui Zhang, Haoyu Wang, Lu Li, Zilin Wang, Zhen Wang, Li Shen,
  Xueqian Wang, and Dacheng Tao.
\newblock Learning better with less: Effective augmentation for
  sample-efficient visual reinforcement learning, 2023.

\bibitem{Marrie_2023_CVPR}
Juliette Marrie, Michael Arbel, Diane Larlus, and Julien Mairal.
\newblock Slack: Stable learning of augmentations with cold-start and kl
  regularization.
\newblock In {\em Proceedings of the IEEE/CVF Conference on Computer Vision and
  Pattern Recognition (CVPR)}, pages 24306--24314, June 2023.

\bibitem{mintun2021on}
Eric Mintun, Alexander Kirillov, and Saining Xie.
\newblock On interaction between augmentations and corruptions in natural
  corruption robustness, 2021.

\bibitem{Modas2021Dec}
Apostolos Modas, Rahul Rade, Guillermo
  Ortiz-Jim{\ifmmode\acute{e}\else\'{e}\fi}nez, Seyed-Mohsen Moosavi-Dezfooli,
  and Pascal Frossard.
\newblock {PRIME: A few primitives can boost robustness to common corruptions}.
\newblock {\em arXiv}, Dec. 2021.

\bibitem{Muller2021Mar}
Samuel~G. M{\ifmmode\ddot{u}\else\"{u}\fi}ller and Frank Hutter.
\newblock {TrivialAugment: Tuning-free Yet State-of-the-Art Data Augmentation}.
\newblock {\em arXiv}, Mar. 2021.

\bibitem{recht2019imagenet}
Benjamin Recht, Rebecca Roelofs, Ludwig Schmidt, and Vaishaal Shankar.
\newblock Do imagenet classifiers generalize to imagenet?, 2019.

\bibitem{Saikia_2021_ICCV}
Tonmoy Saikia, Cordelia Schmid, and Thomas Brox.
\newblock Improving robustness against common corruptions with frequency biased
  models.
\newblock In {\em Proceedings of the IEEE/CVF International Conference on
  Computer Vision (ICCV)}, pages 10211--10220, October 2021.

\bibitem{Soklaski2022Feb}
Ryan Soklaski, Michael Yee, and Theodoros Tsiligkaridis.
\newblock {Fourier-Based Augmentations for Improved Robustness and Uncertainty
  Calibration}.
\newblock {\em arXiv}, Feb. 2022.

\bibitem{strisciuglio2022visual}
Nicola Strisciuglio and George Azzopardi.
\newblock Visual response inhibition for increased robustness of convolutional
  networks to distribution shifts.
\newblock In {\em NeurIPS 2022 Workshop on Distribution Shifts: Connecting
  Methods and Applications}, 2022.

\bibitem{Strisciuglio2020}
Nicola Strisciuglio, Manuel Lopez-Antequera, and Nicolai Petkov.
\newblock Enhanced robustness of convolutional networks with a push--pull
  inhibition layer.
\newblock {\em Neural Computing and Applications}, 32(24):17957--17971, 2020.

\bibitem{Suzuki_2022_CVPR}
Teppei Suzuki.
\newblock Teachaugment: Data augmentation optimization using teacher knowledge.
\newblock In {\em Proceedings of the IEEE/CVF Conference on Computer Vision and
  Pattern Recognition (CVPR)}, pages 10904--10914, June 2022.

\bibitem{10190316}
An Wang, Mobarakol Islam, Mengya Xu, and Hongliang Ren.
\newblock Curriculum-based augmented fourier domain adaptation for robust
  medical image segmentation.
\newblock {\em IEEE Transactions on Automation Science and Engineering}, pages
  1--13, 2023.

\bibitem{Wang2021Oct}
Haotao Wang, Chaowei Xiao, Jean Kossaifi, Zhiding Yu, Anima Anandkumar, and
  Zhangyang Wang.
\newblock {AugMax: Adversarial Composition of Random Augmentations for Robust
  Training}.
\newblock {\em arXiv}, Oct. 2021.

\bibitem{wang2023dfmx}
Shunxin Wang, Christoph Brune, Raymond Veldhuis, and Nicola Strisciuglio.
\newblock {DFM}-x: Augmentation by leveraging prior knowledge of shortcut
  learning.
\newblock In {\em 4th Visual Inductive Priors for Data-Efficient Deep Learning
  Workshop}, 2023.

\bibitem{wang2022frequency}
Shunxin Wang, Raymond Veldhuis, Christoph Brune, and Nicola Strisciuglio.
\newblock Frequency shortcut learning in neural networks.
\newblock In {\em NeurIPS 2022 Workshop on Distribution Shifts: Connecting
  Methods and Applications}, 2022.

\bibitem{Wang2023May}
Shunxin Wang, Raymond Veldhuis, Christoph Brune, and Nicola Strisciuglio.
\newblock {Larger is not Better: A Survey on the Robustness of Computer Vision
  Models against Common Corruptions}.
\newblock {\em arXiv}, May 2023.

\bibitem{Wang2023ICCV}
Shunxin Wang, Raymond Veldhuis, Christoph Brune, and Nicola Strisciuglio.
\newblock What do neural networks learn in image classification? a frequency
  shortcut perspective.
\newblock In {\em Proceedings of the IEEE/CVF International Conference on
  Computer Vision (ICCV)}, pages 1433--1442, October 2023.

\bibitem{Xie_2020_CVPR}
Qizhe Xie, Minh-Thang Luong, Eduard Hovy, and Quoc~V. Le.
\newblock Self-training with noisy student improves imagenet classification.
\newblock In {\em Proceedings of the IEEE/CVF Conference on Computer Vision and
  Pattern Recognition (CVPR)}, June 2020.

\bibitem{XU2023109474}
Qinwei Xu, Ruipeng Zhang, Ziqing Fan, Yanfeng Wang, Yi-Yan Wu, and Ya Zhang.
\newblock Fourier-based augmentation with applications to domain
  generalization.
\newblock {\em Pattern Recognition}, 139:109474, 2023.

\bibitem{Yin2019Jun}
Dong Yin, Raphael~Gontijo Lopes, Jonathon Shlens, Ekin~D. Cubuk, and Justin
  Gilmer.
\newblock {A Fourier Perspective on Model Robustness in Computer Vision}.
\newblock {\em arXiv}, June 2019.

\bibitem{yucel2023hybridaugment}
Mehmet~Kerim Yucel, Ramazan~Gokberk Cinbis, and Pinar Duygulu.
\newblock Hybridaugment++: Unified frequency spectra perturbations for model
  robustness, 2023.

\bibitem{yun2019cutmix}
Sangdoo Yun, Dongyoon Han, Seong~Joon Oh, Sanghyuk Chun, Junsuk Choe, and
  Youngjoon Yoo.
\newblock Cutmix: Regularization strategy to train strong classifiers with
  localizable features, 2019.

\bibitem{zhang2018mixup}
Hongyi Zhang, Moustapha Cisse, Yann~N. Dauphin, and David Lopez-Paz.
\newblock mixup: Beyond empirical risk minimization.
\newblock In {\em International Conference on Learning Representations}, 2018.

\bibitem{Zheng2016Apr}
Stephan Zheng, Yang Song, Thomas Leung, and Ian Goodfellow.
\newblock {Improving the Robustness of Deep Neural Networks via Stability
  Training}.
\newblock {\em arXiv}, Apr. 2016.

\end{thebibliography}
